\documentclass[letterpaper,11pt]{article}
\usepackage{parskip}
\usepackage[english]{babel}
\usepackage[nottoc]{tocbibind}

\usepackage{natbib}

\usepackage{diagbox}
\usepackage{makecell}
\usepackage{booktabs}
\usepackage{tikz}
\usepackage{amsmath}
\usepackage{amsfonts}
\usepackage{amssymb}
\usepackage{amsthm}
\usepackage{url,ifthen}
\usepackage{enumitem}
\usepackage{srcltx}
\usepackage{multirow}
\usepackage{boxedminipage}
\usepackage[margin=1.1in]{geometry}
\usepackage{nicefrac}
\usepackage{xspace}
\usepackage{graphicx}
\usepackage{color}
\usepackage{colortbl}
\usepackage{setspace}
\usepackage{soul}

\usepackage{hyperref}
\usepackage{cleveref}
\usepackage{mathptmx}
\usepackage[most]{tcolorbox}

\definecolor{DarkGreen}{rgb}{0.1,0.5,0.1}
\definecolor{DarkRed}{rgb}{0.5,0.1,0.1}
\definecolor{DarkBlue}{rgb}{0.1,0.1,0.5}
\definecolor{Gray}{rgb}{0.2,0.2,0.2}
\usepackage{listings}
\lstdefinestyle{mystyle}{
    commentstyle=\color{DarkBlue},
    keywordstyle=\color{DarkRed},
    numberstyle=\tiny\color{Gray},
    stringstyle=\color{DarkGreen},
    basicstyle=\footnotesize,
    breakatwhitespace=false,         
    breaklines=true,                 
    captionpos=b,                    
    keepspaces=true,                 
    numbers=left,                    
    numbersep=5pt,                  
    showspaces=false,                
    showstringspaces=false,
    showtabs=false,                  
    tabsize=2
}
\usepackage[small]{caption}
\usepackage{hyperref}
\hypersetup{
    unicode=false,          % non-Latin characters in AcrobatÂ¿s bookmarks
    pdftoolbar=true,        % show Acrobat toolbar?
    pdfmenubar=true,        % show Acrobat menu?
    pdffitwindow=false,      % page fit to window when opened
    pdfnewwindow=true,      % links in new window
    colorlinks=true,       % false: boxed links; true: colored links
    linkcolor=blue,          % color of internal links
    citecolor=blue,        % color of links to bibliography
    filecolor=DarkRed,      % color of file links
    urlcolor=blue,          % color of external links
    pdftitle={},
    pdfauthor={},
}

\def\draft{1}

\def\submit{0}

\ifnum\draft=1 % show authors' note if draft
    \def\ShowAuthNotes{1}
\else
    \def\ShowAuthNotes{0}
\fi

\ifnum\submit=1
\newcommand{\forsubmit}[1]{#1}
\newcommand{\forreals}[1]{}
\else
\newcommand{\forreals}[1]{#1}
\newcommand{\forsubmit}[1]{}
\fi

\ifnum\ShowAuthNotes=1
\newcommand{\authnote}[2]{{ \footnotesize \bf{\color{DarkRed}[#1's Note:
{\color{DarkBlue}#2}]}}}
\else
\newcommand{\authnote}[2]{}
\fi

\newtheorem{theorem}{Theorem}[section]

\newtheorem{lemma}[theorem]{Lemma}

\newtheorem{proposition}[theorem]{Proposition}
\theoremstyle{definition}
\newtheorem{definition}[theorem]{Definition}
\newtheorem{assumption}[theorem]{Assumption}

\newtheoremstyle{example_contd}
{\topsep} {\topsep}%
{}% Body font
{}% Indent amount (empty = no indent, \parindent = para indent)
{\bfseries}% Thm head font
{.}% Punctuation after thm head
{1em}% Space after thm head (\newline = linebreak)
{\thmname{#1} \thmnumber{ #2}\thmnote{#3} (continued)}% Thm head spec

\theoremstyle{example_contd}

\newcommand{\chapterref}[1]{\hyperref[ch:#1]{Chapter~\ref{ch:#1}}}

\newcommand{\claimref}[1]{\hyperref[claim:#1]{Claim~\ref{claim:#1}}}

\newcommand{\corollaryref}[1]{\hyperref[cor:#1]{Corollary~\ref{cor:#1}}}

\newcommand{\definitionref}[1]{\hyperref[def:#1]{Definition~\ref{def:#1}}}

\newcommand{\equationref}[1]{\hyperref[eq:#1]{Equation~\ref{eq:#1}}}

\newcommand{\factref}[1]{\hyperref[fact:#1]{Fact~\ref{fact:#1}}}

\newcommand{\figureref}[1]{\hyperref[fig:#1]{Figure~\ref{fig:#1}}}

\newcommand{\tableref}[1]{\hyperref[tab:#1]{Table~\ref{tab:#1}}}

\newcommand{\itemref}[1]{\hyperref[item:#1]{Item~(\ref{item:#1})}}

\newcommand{\lemmaref}[1]{\hyperref[lem:#1]{Lemma~\ref{lem:#1}}}

\newcommand{\propref}[1]{\hyperref[prop:#1]{Proposition~\ref{prop:#1}}}

\newcommand{\propositionref}[1]{\hyperref[prop:#1]{Proposition~\ref{prop:#1}}}

\newcommand{\remarkref}[1]{\hyperref[rem:#1]{Remark~\ref{rem:#1}}}

\newcommand{\sectionref}[1]{\hyperref[sec:#1]{Section~\ref{sec:#1}}}

\newcommand{\theoremref}[1]{\hyperref[thm:#1]{Theorem~\ref{thm:#1}}}

\usepackage[utf8]{inputenc}
\usepackage[T1]{fontenc}
\usepackage{kpfonts}
\usepackage{microtype}

\newcommand{\Psymb}{\mathbb{P}}

\DeclareMathOperator*{\ProbOp}{\Psymb r}
\renewcommand{\Pr}{\ProbOp}

\newcommand{\Prob}[1]{\Pr\left\{ #1 \right\}}

\renewcommand{\hat}{\widehat}

\newcommand{\cD}{{\cal D}}

\newcommand{\cI}{{\cal I}}

\newcommand{\cS}{{\cal S}}
\newcommand{\cT}{{\cal T}}

\newcommand{\cX}{{\cal X}}
\newcommand{\cY}{{\cal Y}}
\newcommand{\cZ}{{\cal Z}}

\renewcommand{\leq}{\leqslant}
\renewcommand{\le}{\leqslant}
\renewcommand{\geq}{\geqslant}
\renewcommand{\ge}{\geqslant}
\usepackage{bm}

\newcommand{\ignore}[1]{}

\renewcommand{\epsilon}{\varepsilon}

\newcommand{\remove}[1]{}

\usepackage{color-edits}
\addauthor{hui}{cyan}

\def\sc{\mathbin{;}}

\newcommand{\bE}{\mathbb{E}}

\newcommand{\bN}{\mathbb{N}}
\newcommand{\bP}{\mathbb{P}}

\newcommand{\bR}{{\mathbb R}}

\newcommand{\lambdasp}{\lambda_*^+}

\usepackage{bm}
\usepackage{amsfonts}
\usepackage[utf8]{inputenc} % allow utf-8 input
\usepackage[T1]{fontenc}    % use 8-bit T1 fonts
\usepackage{hyperref}       % hyperlinks
\usepackage{url}            % simple URL typesetting
\usepackage{booktabs}       % professional-quality tables
\usepackage{amsfonts}       % blackboard math symbols
\usepackage{nicefrac}       % compact symbols for 1/2, etc.
\usepackage{microtype}      % microtypography
\usepackage{xcolor}         % colors

\usepackage{algorithm}
\usepackage{algpseudocode}
\usepackage{diagbox}
\usepackage{makecell}
\usepackage{booktabs}
\usepackage{tikz}
\usepackage{amsmath}
\allowdisplaybreaks
\usepackage{amssymb}
\usepackage{amsthm}
\usepackage{url,ifthen}
\usepackage{enumitem}
\usepackage{srcltx}
\usepackage{multirow}
\usepackage{boxedminipage}
\usepackage{nicefrac}
\usepackage{xspace}
\usepackage{graphicx}
\usepackage{color}
\usepackage{colortbl}
\usepackage{setspace}
\usepackage{soul}
\usepackage{graphicx}
\usepackage{wrapfig}
\usepackage{subcaption}
\usepackage{enumitem}
\usepackage{cleveref}
\definecolor{DarkGreen}{rgb}{0.1,0.5,0.1}
\definecolor{DarkRed}{rgb}{0.5,0.1,0.1}
\definecolor{DarkBlue}{rgb}{0.1,0.1,0.5}
\definecolor{Gray}{rgb}{0.2,0.2,0.2}
\title{Performative Risk Control: Calibrating Models for Reliable Deployment under Performativity}

\usepackage{authblk}

\author[1]{Victor Li}
\author[2]{Baiting Chen}
\author[3]{Yuzhen Mao}
\author[1]{Qi Lei}
\author[4]{Zhun Deng}

\affil[1]{New York University}
\affil[2]{University of California, Los Angeles}
\affil[3]{Simon Fraser University}
\affil[4]{University of North Carolina at Chapel Hill}

\date{}
\begin{document}

\maketitle

\begin{abstract}
Calibrating blackbox machine learning models to achieve risk control is crucial to ensure reliable decision-making. A rich line of literature has been studying how to calibrate a model so that its predictions satisfy explicit finite-sample statistical guarantees under a \textit{fixed}, \textit{static}, and unknown data-generating distribution. However, prediction-supported decisions may influence the outcome they aim to predict, a phenomenon named \textit{performativity} of predictions, which is commonly seen in social science and economics. In this paper, we introduce \textit{Performative Risk Control}, a framework to calibrate models to achieve risk control under performativity with provable theoretical guarantees. Specifically, we provide an iteratively refined calibration process, where we ensure the predictions are improved and risk-controlled throughout the process. We also study different types of risk measures and choices of tail bounds. Lastly, we demonstrate the effectiveness of our framework by numerical experiments on the task of predicting credit default risk. To the best of our knowledge, this work is the first one to study statistically rigorous risk control under performativity, which will serve as an important safeguard against a wide range of strategic manipulation in decision-making processes. 
\end{abstract}

\section{Introduction}\label{sec:intro}

We have entered an era characterized by the ubiquitous deployment of increasingly complex models, such as large language models (LLMs) with billions of parameters \citep{achiam2023gpt,team2023gemini,team2024gemini}. These models play important roles in our daily lives, informing us, shaping our opinions, and deciding allocation of societal resources. However, in most scenarios, these models can only be treated as a blackbox, either because they are too complex to understand or simply because their details are kept private by companies. Given their influence in today's society and their blackbox nature, it is urgent to derive new tools to ensure reliable deployment.

A recent line of work \citep{angelopoulos2022ltt,bates2021distribution,angelopoulos2022conformal} has been investigating methods to calibrate the predictions of blackbox machine learning models, generally known as (conformal) risk control. The goal is to ensure reliable deployment by satisfying explicit, finite-sample statistical guarantees (either with high probability or marginally) for controlling risk. The introduced frameworks are lightweight, agnostic to the data-generating distribution, and do not require model refitting. Specifically, in (conformal) risk control, a blackbox model $f$ is given to us, and we need to post-process it using calibration data in order to make our final predictions. The calibration process is governed by a low-dimensional parameter $\lambda$ (e.g., a threshold). For example, \cite{bates2021distribution} and \cite{angelopoulos2022conformal} study false negative rate (FNR) control in tumor segmentation, where $f:\cX\mapsto [0,1]^{d\times d}$ takes an image $x$ as input and outputs scores for all the $d\times d$ pixels in image $x$. For a given $\lambda$, one can produce $\cT_\lambda(x):=\{i: f(x)_i \ge 1-\lambda\}$ to include all the pixels with "high" scores. The aim is to control the expected loss:
$$\bE_{(x,y)\sim\cD}\left[1-\frac{|y\cap \cT_\lambda(x)|}{|y|}\right],$$ where label $y$ is the set of pixels truly containing poly segments, $f(x)_i$ is the $i$-th coordinate of $f(x)$, and $\cD$ is the underlying data-generating distribution.

Nevertheless, in many applications, predictions of machine learning models impact the outcome to predict, a phenomenon known as \textit{performativity} \citep{perdomo2020performative}. For instance, a bank's loan policies may impact the purchasing patterns of the population, which in turn change the features used by the bank in predicting default risk, thus completing a cyclical relationship where predictions and data evolve together. So far, the important topic of risk control on blackbox models under performativity eludes the literature. In light of this, we \textit{\textbf{initate}} the study of \textit{Performative Risk Control} (PRC).

\textbf{Formalizing our goal.} We post-process the outputs of the pre-trained model $f(x)$ in $\cY$ to generate predictions $\cT_\lambda(x)\in \cY'$ indexed by a scalar  threshold $\lambda$. Here, $\cT_\lambda$ could either be standard predictions, i.e., $\cY'=\cY$, or prediction sets, i.e., $\cY'=2^{\cY}$. In this paper, we mainly study the expected risk. However, since we further consider the performativity of predictions, the measure of our interest is mainly the \textit{performative expected risk} defined as
$$R(\lambda):=\bE_{(x,y)\sim\cD(\lambda)}\left[ \ell (y, \cT_\lambda(x))\right],$$
where $\cD(\lambda)$ is the distribution induced by the decision threshold $\lambda$ and is typically unknown. In Sec.~\ref{sec:ext}, we also discuss a broader class of measures beyond expected risk. By carefully setting the parameter $\lambda$, we control a user-chosen error rate, regardless of the quality of
$f$. Specifically, our goal is to calibrate $f$ on a \textit{calibration} dataset $\cI_{\text{cal}}$ (specified later in Sec.~\ref{subsec:procedure}) and achieve the following:
\begin{definition}\label{def:prc}
Let $\hat\lambda\in \Lambda$ be the threshold obtained by calibrating on $\cI_{\text{cal}}$. We say that risk $R(\hat\lambda)$ is $(\alpha,\delta)$-performative-risk-controlled if 
$$\bP(R(\hat\lambda)\le \alpha)\ge 1-\delta,$$
\end{definition}
where the randomness is taken on $\cI_{\text{cal}}$. Here, $\alpha\ge 0$ and $\delta\in (0,1)$ are both specified by the user.

\begin{figure}
    \centering
    \includegraphics[width=\linewidth]{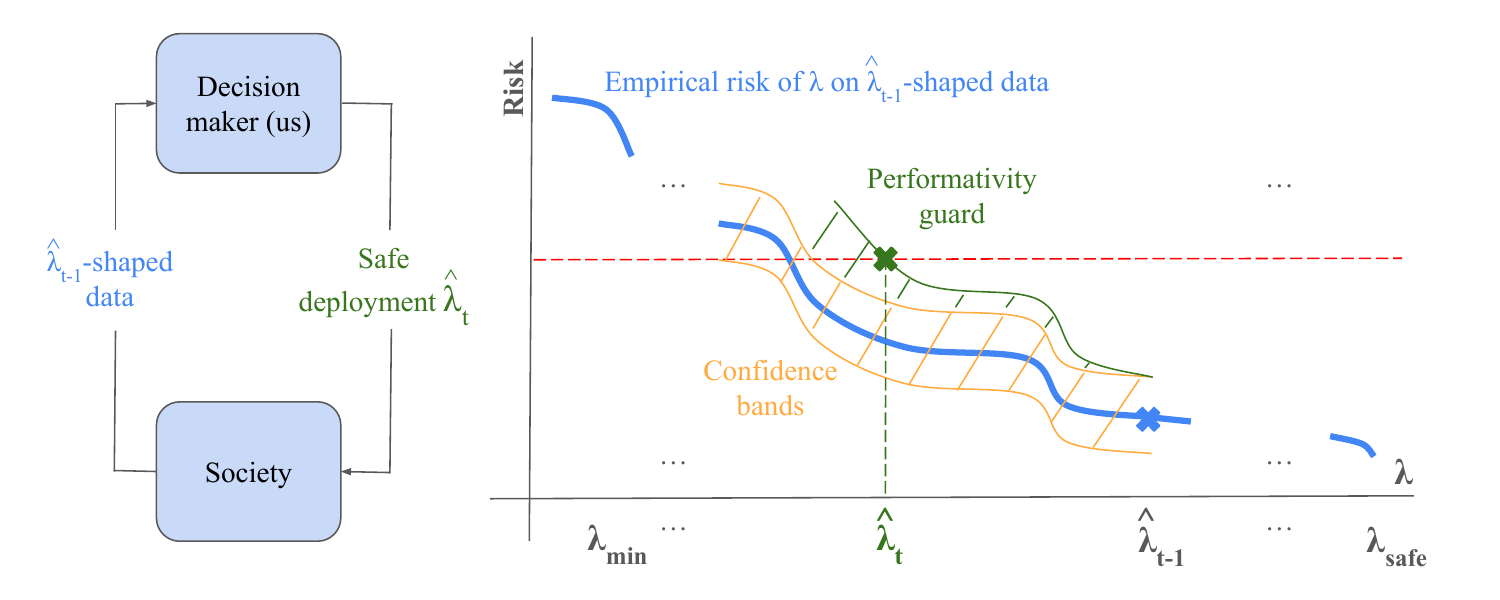}
    \caption{An overview of the Performative Risk Control (PRC) framework. \textbf{Left}: We, the decision maker, use samples from past deployments to determine a $\hat\lambda_t$ that is safe to deploy in the performative environment. \textbf{Right}: We do so by defensively choosing a $\hat\lambda_t$, careful to consider generalization error due to sampling and performative error due to deploying in the new distribution induced by $\hat\lambda_t$, as further explained in Sec.~\ref{sec:prc}.}
    \label{fig:main}
\end{figure}

\textbf{Our contribution.} Our central contribution is to introduce an iteratively refined procedure to select risk-controlling decision thresholds in a performative environment, as illustrated informally in Fig.~\ref{fig:main}. Our framework greatly generalizes the previous line of work on risk control for a static distribution to a dynamic and model-dependent distribution. In particular, our work can handle strategic manipulation of input distributions in decision-processes. Our framework satisfies explicit finite-sample statistical guarantees, and we present experimental results that highlight its practical utility. To the best of our knowledge, we are the first to study risk control under performativity, and we believe our work will serve as an important safeguard in a wide range of applications in social science and economics.

\subsection{Related Work}

\textbf{Performative prediction.} Follow-up works on performative prediction, first pioneered by \citet{perdomo2020performative}, have extended the framework to various settings. \citet{miller2021outside} provides conditions for which performatively optimal model parameters can be found, while \citet{brown2022performative} considers the history of model deployments when modeling performative distribution shifts. \citet{jagadeesan2022regret} and \citet{chen2023performative} cast performativity into a bandit problem, with the assumption of additional knowledge through access to the samples. In these settings, the goal is to find a setting of parameters that in some way minimizes an expected loss and obtain a performative optimal point $\theta_{\text{PO}} = \arg\min_{\theta} \; \mathbb{E}_{z \sim \mathcal{D}(\theta)} \, \ell(z; \theta).$ However, due to the unknown $\cD(\theta)$, this task is intractable, so much focus has been devoted to studying performative stability, i.e., finding $\theta_{\text{PS}} = \arg\min_{\theta} \; \mathbb{E}_{z \sim \mathcal{D}(\theta_{\text{PS}})} \, \ell(z; \theta).$ In our setting, our objective is instead to \textit{control} the expected loss, i.e., find the most aggressive setting of the threshold $\hat\lambda$ while guaranteeing $R(\hat\lambda)$ is risk controlled. To the best of our knowledge, risk control in the performative setting has yet to be explored. For an overview on performative prediction, \citet{hardt2023performative} serves as a solid reference.

\textbf{Distribution-free uncertainty quantification.} Conformal prediction aims to endow black-box models with rigorous, finite-sample statistical guarantees \citep{vovk2005algorithmic, shafer2008atutorial,angelopoulos2021gentle}. Conformal risk control is an extension to this framework that handles a more general class of loss functions \citep{angelopoulos2022ltt,angelopoulos2022conformal}. These works are concerned with controlling the expected risk at some threshold $\alpha$; further work by \citet{snell2022quantile} extends this framework to more general quantile risks.

\textbf{Risk control under distribution shift.} The canonical conformal prediction procedure relies on the assumption that the calibration and test data are i.i.d., or more generally, \textit{exchangeable}. Because this assumption often does not hold in real-world environments, statistical extensions of conformal prediction have been developed to incorporate distribution shift. \citet{tibshirani2019conformal} deals with covariate shift, while \citet{gibbs2021adaptive} handles distribution shift in the online setting by continuously re-estimating a single parameter to achieve exact coverage. \citet{barber2023conformal} generalizes conformal prediction to account for nonexchangeability through a procedure that weighs points in the calibration data higher if they are closer in distribution to the test data. Other works focus on making coverage robust to a set of possible test distributions, e.g., those in an $f$-divergence ball around the training distribution \citep{cauchois2024robust} and those due to feature alterations by strategic agents after model deployment \citep{csillag2024strategicconformalprediction}.
\section{Setup}\label{sec:setup}
This section introduces our setting and goals and their significance more formally.

\subsection{Setting and Notation}
For $K\in \bN_+$, we use $[K]$ to denote $\{1,2,\cdots,K\}$. For simplicity, we denote $\cZ=\cX\times \cY$ and $z=(x,y)$. Recall that we have a prediction function indexed by a threshold $\lambda$, where for a fixed $\lambda\in\Lambda$, $\cT_{\lambda}$ is a function mapping $\cX$ to $\cY'$. Here, $\cY'$ could either be $\cY$ or $2^\cY$. We further use $\ell(z,\lambda)$ short for $\ell(y,\cT_\lambda(x))$.

\noindent\textbf{Monotonicity.} Throughout the paper, we consider $\ell(z,\lambda)$ to be \textit{non-increasing} with respect to $\lambda$ for any $z\in\cZ$. This is mainly following and motivated by previous work in standard risk control \citep{bates2021distribution}. On one hand, one can think of the case that $\cT_\lambda$ is a set-valued function as in \citet{bates2021distribution,angelopoulos2022conformal}, that 
$$\lambda_1 < \lambda_2\Rightarrow \cT_{\lambda_1}(x)\subseteq \cT_{\lambda_2}(x).$$
A larger set includes more candidates and provides a larger tolerance region that will decrease the loss values of loss functions such as the classical one corresponding to tolerance region, i.e., $\ell(y,\cS)=\bm{1}\{y\notin\cS\}$, which satisfies $\cS\subset \cS'\Rightarrow \ell(y,\cS)\ge \ell(y,\cS').$ On the other hand, when $\cT_\lambda(x)$ is a standard prediction, larger $\lambda$ indicates a more conservative but safer decision. For the example of credit loans, $\cT_\lambda(x)=\bm{1}\{f(x)\leq1-\lambda\}$ may indicate whether an applicant gets credit ($1$ means they got it). With increasing $\lambda$, the credit firm narrows the set of credit users, reducing the firm's risk and leading to a decreasing loss function $\ell(z,\lambda)$ w.r.t the threshold $\lambda$.

\noindent\textbf{Setup for $\Lambda$.} We consider the case where there exists a \textit{\textbf{safe}} threshold $\lambda_{\text{safe}}$ such that $\ell(z,\lambda_{\text{safe}})=0$, and we consider $\Lambda=[\lambda_{\text{min}},\lambda_{\text{safe}}]$ for $\lambda_{\text{min}}\in\bR$ \footnote{Our setting follows \cite{bates2021distribution}, in which the loss value becomes trivially too large at $\lambda_{\text{min}}$ and too small at $\lambda_\text{safe}$. This holds generally in risk control literature.}. For the example of $\cT_\lambda(x)=\bm{1}\{f(x)\leq1-\lambda\}$ above, we can choose $\lambda_{\text{safe}}=1$ and in the example of detecting tumor segmentation in \cite{bates2021distribution}, we can choose $\lambda_{\text{safe}}$ large enough to make $\cS_{\lambda_{\text{safe}}}$ include all possible labels, so that $\ell(y,\cS_{\lambda_{\text{safe}}})=\bm{1}\{y\notin\cS_{\lambda_{\text{safe}}}\}=0$.

\subsection{Our Desiderata and Significance}
As mentioned in the introduction, our main goal is to select an \textit{appropriate} threshold $\hat\lambda$ by using the calibration dataset $\cI_{\text{cal}}$, such that $$\bP(R(\hat\lambda)\le \alpha)\ge 1-\delta.$$
However, there might be multiple choices for $\hat\lambda$ and how to define ``appropriate" under performativity is nuanced. We further illustrate our extra desiderata here. 

%======================================
\noindent\textbf{Achieve user-specified conservativeness.} Our setting prescribes that a larger $\lambda$ leads to being more conservative. As an extreme case, choosing a $\lambda$ large enough to include all the pixels in the tumor segmentation example \citep{bates2021distribution} can guarantee zero loss, thus satisfying $(\alpha,\delta)$-performative-risk-controlled for all $\alpha\ge 0$. But this safe-guard choice of $\lambda$ is not useful in practice since it sacrifices too much utility in providing useful class information. This illustration demonstrates that a user of this framework may further desire to choose a $\lambda$ that is less conservative. To formalize this, we wish the final choice of $\hat\lambda$ also satisfies
$$R(\hat\lambda)\ge\alpha-\Delta\alpha$$
for a small user-specified $\Delta \alpha$. In particular, we show that our framework allows users to choose $\Delta \alpha$ such that $\Delta\alpha\rightarrow 0$ as $n\rightarrow \infty$. This means that our framework can provide a $\hat\lambda$ also satisfying a tight lower bound guarantee for the performative risk, i.e., $\hat\lambda$ is the optimal and least conservative in an \textit{asymptotic} sense.

%======================================
\noindent\textbf{Safe at anytime.} Similar to the literature of performative prediction \citep{perdomo2020performative}, our algorithm is an iterative algorithm (see Sec.\ref{sec:prc} for details). On a high-level summary, our algorithm guarantees that the $\hat\lambda_t$ updated in each iteration is improving, i.e., $\hat\lambda_t<\hat\lambda_{t-1}$, until we find our final $\hat\lambda$ that satisfies the user-specified conservativeness. However, we don't want any of the $\hat\lambda_t$'s to be chosen too aggressively, resulting in a threshold too small to be risk-controlled. This is especially important in policy-making, where a policy's risk must be closely monitored as its usage is incrementally increased. Thus, we further require that the trajectory of the $\hat\lambda_t$'s is safe for all $t\in [T]$:
$$R(\hat\lambda_t)\le \alpha,$$
where $T$ is the iteration of return in our algorithm.
\section{Performative Risk Control}\label{sec:prc}
In this section, we describe main results for PRC. We provide an iterative procedure to compute a series of safe thresholds. Specifically, we demonstrate that our framework provides statistical guarantees on the ability to maintain risk control at user-specified level $\alpha$ throughout the iterative process, as well as obtain $\hat\lambda_T$ achieving user-specified tightness at $\alpha-\Delta\alpha$.

\subsection{Our Procedure}\label{subsec:procedure}

To give a high-level summary, we provide an iterative procedure to obtain less conservative thresholds at every iteration. Specifically, not until the process is terminated, at time $t\ge 1$, we make progress by choosing a less conservative threshold $\hat\lambda_t$ such that $\hat\lambda_t< \hat\lambda_{t-1}$. Meanwhile, we ensure that $R(\hat\lambda_t)\le \alpha$ for all $t\in[T]$ with high probability. Eventually, our procedure will end up with a $\hat\lambda_T$ that achieves user-specified tight risk control. 

For user-specified $\delta, \alpha, \Delta\alpha$, the PRC procedure is outlined below. We start with setting $\lambda_0=\lambda_{\text{safe}}$ and then do the following process:
%\begin{enumerate}
%\item We start with setting $\lambda_0=\lambda_{\text{safe}}$ and collect a calibration dataset $\cI^{0}_{\text{cal}}=\{(x_{0,i},y_{0,i})\}_{i=1}^{n_0}\in\cX^{n_0}\times\cY^{n_0}$ i.i.d. drawn from $\cD(\lambda_0)$. We calculate 
%$$\hat\lambda_1=\inf \{\lambda\in \Lambda\mid V(\lambda_0,\lambda,\delta)\le \alpha\}$$
%for a function $V$. If $\hat\lambda_1\le \lambda_0-\Delta\lambda$ for a progress measure $\Delta\lambda$, which means the $\hat\lambda_1$ makes big enough progress, then continue to Step 2. Otherwise, return $\hat\lambda_1$.

\begin{tcolorbox}[title=Performative Risk Control (PRC)]

For $t\ge 1$, we sample $\cI^{t}_{\text{cal}}=\{(x_{t-1,i},y_{t-1,i})\}_{i=1}^{n_{t-1}}\in\cX^{n_{t-1}}\times\cY^{n_{t-1}}$ i.i.d. drawn from $\cD(\hat \lambda_{t-1})$ (here, $\hat\lambda_0=\lambda_0$) and calculate 
$$\hat\lambda_t=\min\Big\{\inf \{\lambda\in \Lambda\mid V(\hat\lambda_{t-1},\lambda,\delta)\le \alpha\},\hat\lambda_{t-1}\Big\}.$$
If $\hat\lambda_t< \hat\lambda_{t-1}-\Delta\lambda$ for a progress measure $\Delta\lambda$, which means $\hat\lambda_t$ makes big enough progress over $\hat\lambda_{t-1}$, then repeat the process with $t+1$. Otherwise, return $\hat\lambda_t$.
%\end{enumerate}
\end{tcolorbox}

We specify details of $V$ and $\Delta\lambda$ later. Notice that the calibration dataset $\cI_{\text{cal}}$ in Def.~\ref{def:prc} is the union of calibration sets throughout the process, i.e., $\cI_{\text{cal}}=\cup_t\cI^t_{\text{cal}}$.  The formal and complete algorithm is listed in Alg.~\ref{alg:main}.

\subsection{Technical Details} 
In this subsection, we spell out details regarding various parameters and functions in our procedure. Before that, we need to state our key assumption.  Intuitively, if the distribution shift of $\cD(\lambda)$ between implementing different $\lambda$'s is too dramatic, it will be hard to achieve performative risk control. Thus, we impose our key distributional assumptions below, which is similar to the one proposed in classical performative prediction \citep{perdomo2020performative} but more general.

%\paragraph{Assumptions.} , we impose 

\begin{definition}[$(\gamma,p,g)$-sensitivity]
Let $g: \mathcal{Z}\rightarrow\mathbb{R}$ be a deterministic function, and denote $\mathcal{D}_g(\lambda)$ as the induced distribution of $g(z)$ where $z\sim\mathcal{D}(\lambda)$. A distribution map $\mathcal{D}(\cdot)$ is $(\gamma,p,g)$-sensitive if for all $\lambda_1,\lambda_2\in\Lambda$,
$$W_p(\mathcal{D}_g(\lambda_1), \mathcal{D}_g(\lambda_2))\leq \gamma|\lambda_1-\lambda_2|$$
    where $W_p$ denotes the $p$-Wasserstein distance.
\label{definition:sensitivity}
\end{definition}

The following assumption applies $(\gamma,p,g)$-sensitivity to our setting. Intuitively, it means that the distribution of loss functions arising from samples drawn from similar thresholds are also similar.

\begin{assumption}[Lipschitz Distribution Mapping] \label{assumption:sensitivity}
We assume that $\mathcal{D}(\cdot)$ is $(\gamma,p,\ell(\cdot,\lambda))$-sensitive for all $\lambda\in\Lambda$.
\end{assumption}
We remark here that the distributional assumption in \cite{perdomo2020performative} is a special case of our assumption, since it could be viewed as assuming $\cD(\cdot)$ is $(\tilde\gamma,1,\ell(\cdot,\lambda))$-sensitive for a Lipchitz loss function $\ell(z,\lambda)$ w.r.t $z$ and an appropriately chosen $\tilde\gamma$ in our setting. A more detailed discussion is deferred to App.~\ref{app:sensitivity}.

Now, we are ready to state further technical details. Without loss of generality, we can consider $n_t=n$ for all $t \ge 0$ since we can denote $n$ as the minimum of $\{n_t\}_t$ presented in our procedure, and all the results below still hold.

\noindent\textbf{Choice of $V$.} We denote $\hat R_n(\lambda,\lambda'):=\frac{1}{n}\sum_{i=1}^n \ell(z_i(\lambda),\lambda')$ to be the empirical mean on the loss of samples $\{z_i(\lambda)\}_{i=1}^n \overset{\text{i.i.d.}}{\sim} \mathcal{D}(\lambda)$ when evaluated against the threshold $\lambda'$. Thus, $\cI^t_{\text{cal}}$ could also be denoted as $\{z_i(\hat\lambda_t)\}_{t=1}^{n}$. We further denote $R(\lambda,\lambda'):=\bE_{z\sim \cD(\lambda)}\ell(z,\lambda')$. 

To construct $V$, we require knowledge of a \textit{confidence width} $c(n,\delta')$ defining pointwise confidence bounds $\hat R^\pm_n(\lambda,\lambda',\delta'):= \hat R_n(\lambda,\lambda') \pm c(n,\delta')$ that satisfy the following property for $\hat\lambda_{t-1}$ and $\lambda'\geq\hat\lambda_{t}$ for all $t\geq1$ encountered in the procedure:
\begin{equation}
\label{eq:pointwise-bound}
\mathbb{P}\left(\hat R_n^-(\hat\lambda_{t-1},\lambda',\delta') \leq  R(\hat\lambda_{t-1},\lambda') \leq \hat R^+_n(\hat\lambda_{t-1},\lambda',\delta') \right) \geq 1-\delta'.
\end{equation}
We discuss choices for $c(n,\delta')$--including derivations from Hoeffding's inequality, Bernstein's inequality, Hoeffding-Bentkus's inequality, and the central limit theorem--in App.~\ref{app:bounds}. We are further often interested in the risk profile of $R(\hat\lambda_{t},\cdot)$, whose difference from $R(\hat\lambda_{t-1},\cdot)$ can be bounded via sensitivity. We call this difference the \textit{performative error}, or the error that arises from evaluating on two different distributions. In a nutshell, we hope to update $\hat\lambda_{t-1}\rightarrow\hat\lambda_{t}$ by accounting for increases in the risk due to generalization and performative error, which leads to the following form:
\[
V(\hat\lambda_{t-1},\lambda,\delta):=\underbrace{\hat R_n(\hat\lambda_{t-1},\lambda)}_{\text{empirical loss}} +
\underbrace{c(n,\delta/\tilde T)}_{\text{confidence width}} +
\underbrace{\tau(\hat\lambda_{t-1}-\lambda)}_{\text{performativity guard}}. 
\]
Here, $\tilde T$ is an upper bound on the number of iterations required in our procedure, not to be confused with $T$ and discussed in (c). Meanwhile, the user-chosen parameter $\tau>0$ guards against performative effects; we require $\tau\geq \gamma$. While $\gamma$ is typically unknown, one can estimate its potential range or magnitude based on domain knowledge of the problem at hand to then choose a high enough $\tau$.

\noindent\textbf{Choice of $\Delta\lambda$.} We choose 
$\Delta\lambda=\frac{1}{2\tau}\left(\Delta \alpha-2c(n,\delta/\tilde T)\right).$
$\Delta\lambda$ represents the minimum decrease in $\lambda$ required for the procedure to continue iterating. We require $\Delta\lambda>0$ to ensure that our procedure returns in a finite number of iterations.

\noindent\textbf{Number of iterations $T$ vs. guaranteed convergence $\tilde T$.} Because we require $\hat\lambda_t < \hat\lambda_{t-1}-\Delta\lambda$ and the interval $[\lambda_\text{min},\lambda_\text{safe}]$ is finite, our algorithm terminates in at most $\lceil(\lambda_\text{safe}-\lambda_\text{min}) / \Delta\lambda \rceil$ iterations. Hence, by choosing $\tilde T\geq \lceil(\lambda_\text{safe}-\lambda_\text{min}) / \Delta\lambda \rceil$, we can guarantee convergence, i.e. $T\leq \tilde T$ always. Further, note the cyclical dependence between $\tilde T$ and $\Delta\lambda$. We can find solutions $(\tilde T, \Delta\lambda)\in \bN\times\bR$ as follows: for a given $\tilde T$, we solve $\Delta\lambda= \frac{\Delta \alpha-2c(n,\delta/\tilde T)}{2\tau}$ and check if $\Delta\lambda \geq (\lambda_\text{safe}-\lambda_\text{min}) / \tilde T$. If there are no such solutions $(\tilde T, \Delta\lambda)$, the algorithm cannot guarantee that its sequentially produced solutions are safe; in this case, we just return $\lambda_\text{safe}$.

\noindent\textbf{Details about $\Delta\alpha$.} 
Suppose we let $\tilde T, n\rightarrow \infty$ with $\tilde T=Cn^r$ (for some constant $C$ and $r\geq1/2$) and $\Delta\lambda=(\lambda_\text{safe}-\lambda_\text{min})/\tilde T$. Then, $\Delta\alpha=O(\ln(n)/\sqrt{n})$. We leave this derivation to App.~\ref{app:asymptotics}.

We summarize all the discussions above with the following detailed and complete algorithm in Alg.~\ref{alg:main}.

\begin{algorithm}
\caption{Performative Risk Control}
\begin{algorithmic}[1] % The [1] adds line numbers
\Require $\alpha$, $\Delta\alpha$, $\delta$, $n$, loss $\ell$, distribution mapping $\mathcal{D}(\cdot)$, $\lambda_\text{min}$, $\lambda_\text{safe}$, safety parameter $\tau$.
\Ensure Output $\hat\lambda$
\State Initialize $\lambda_0\gets\lambda_\text{safe}$
\State Jointly solve $(\tilde T, \Delta\lambda)\in \bN\times\bR$ s.t. $2\tau\Delta\lambda=\Delta\alpha-2c(n,\delta/\tilde T)$ and $\Delta\lambda\geq(\lambda_\text{safe}-\lambda_\text{min})/\tilde T$. Break ties by choosing the solution with the minimum $\tilde T$. If there are no solutions, return $\lambda_\text{safe}$.\label{alg:joint-solve}
\For{$t = 1$ to $\tilde T$}
    \State Receive samples $\{z_{t-1,i}\}_{i=1}^n \overset{\text{i.i.d.}}{\sim} \mathcal{D}(\hat\lambda_{t-1})$ where $\hat\lambda_0=\lambda_0$
    \State Set $\hat\lambda_t\gets\min\Big\{\inf \{\lambda\in \Lambda\mid V(\hat\lambda_{t-1},\lambda,\delta)\le \alpha\},\hat\lambda_{t-1}\Big\}$ \label{alg:line-sett}
    \If{ $\hat\lambda_{t}\geq\hat\lambda_{t-1}-\Delta\lambda$ }
        \State \label{alg:return} \Return $\hat\lambda_{t}$
    \EndIf
\EndFor
\end{algorithmic}
\label{alg:main}
\end{algorithm}

\subsection{Theoretical Guarantees}
In this subsection, we state our main theoretical result. By the following theorem, we demonstrate how our procedure can realize the promise of achieving user-specified conservativeness and being safe and reliable at anytime.

\begin{theorem}[Risk control of PRC]
Under Assumption~\ref{assumption:sensitivity}, where the distribution mapping $\mathcal{D}(\cdot)$ is $(\gamma,1,\ell(\cdot,\lambda))$-sensitive for all $\lambda\in\Lambda$, if the loss function $\ell(z,\lambda)$ is continuous in $\lambda$ for all $z$, $\tau\geq \gamma$, and the initial joint solve of $(\tilde T,\Delta\lambda)$ produces at least one value, Algorithm \ref{alg:main} guarantees that with probability $1-\delta$, the following three conditions are met simultaneously:
(i) \textbf{Safety in the iterative process:} $R(\hat\lambda_{t-1},\hat\lambda_{t})\leq\alpha$ for $1\leq t\leq T$.
(ii) \textbf{Safety at anytime: } $R(\hat\lambda_t)\leq\alpha$ for $0\le t\le  T$.
(iii) \textbf{Tightness of $\hat \lambda_T$}: $R(\hat\lambda_T)\geq\alpha-\Delta\alpha$ if $T\ge 1$. 
\label{thm:risk-control}
\end{theorem}

%During the exploration phase, PRC 's iterative procedure guarantees safety. Finally, PRC determines a deployment level $\hat \lambda$ that maximizes AI utilization within the human-AI collaboration while satisfying risk control.

\section{Extension}
\label{sec:ext}

In this section, we discuss how PRC can be extended to handle quantile-based risk measures, which is commonly used in quantify tail risk and widely used in mathematical finance. In particular, both conditional variance-at-risk (CVaR) and expected risk belong to quantile-based risk measures. Denote the cumulative distribution function (CDF) of losses for $z\sim\cD(\lambda)$ and evaluation threshold $\lambda'$ as $F(w \mathbin{;} \lambda,\lambda') := \Prob(\ell(z,\lambda')\leq w)$. Further, recall that the inverse of CDF $F$ is defined as $F^{-1}(p):=\inf\{ x: F(x) \geq p \}$, which leads to the following definition.

\begin{definition}[Quantile-based risk measure]\label{def:quantile-risk-measure}
Let $\psi(p)$ be a weighting function such that $\psi(p)\geq 0$ and $\int_0^1 \psi(p) dp=1$. The quantile-based risk measure defined by $\psi$ is
$$R_{\psi}(F ) = \int_0^1 \psi(p)F^{-1}(p) \, dp$$
\end{definition}

From now on, we refer to the quantile-based risk as $R_\psi(\lambda,\lambda'):=R_\psi(F(\cdot\sc\lambda,\lambda'))$. The analogy to the expected risk case is made clear when $\psi(p)=1$. Indeed,
$$
R_{\psi=1}(\lambda,\lambda'):= \int_0^1 F^{-1}(p \sc \lambda, \lambda') \, dp = \bE_{z\sim\mathcal{D}(\lambda)} \ell(z,\lambda')=R(\lambda,\lambda')
$$

Other values of $\psi(p)$ allow us to control the $\beta$-VaR (e.g., the 90th percentile of losses) and $\beta$-CVaR (e.g., the average of the worst $10\%$ of losses). To extend Thm.~\ref{thm:risk-control}, we need to construct a confidence width $c(n,\delta')$. To do so, note that we have knowledge of the empirical loss CDF of the samples from iteration $t-1$ evaluated against any threshold $\lambda'$: $$\hat F_n(w\sc\hat\lambda_{t-1},\lambda'):=\frac{1}{n}\sum_{i=1}^n \mathbf{1}\{\ell(z_i(\hat\lambda_{t-1}),\lambda') \leq w\}.$$

We follow the technique in \cite{snell2022quantile} and \cite{zollo2023prompt} to create an LCB $\hat F_{L,n}(\cdot \mathbin{;} \hat\lambda_{t-1},\lambda',\delta')$ and UCB $\hat F_{U,n}(\cdot \mathbin{;} \hat\lambda_{t-1},\lambda',\delta')$ based on the empirical CDF, where
$$\mathbb{P}\left( \hat F_{L,n}(w \mathbin{;} \hat\lambda_{t-1},\lambda',\delta')\leq  F(w \mathbin{;} \hat\lambda_{t-1},\lambda')\leq \hat F_{U,n}(w \mathbin{;} \hat\lambda_{t-1},\lambda',\delta')   \right) \geq 1-\delta' \quad \forall w$$
Next, similar to how $R_\psi(\lambda,\lambda')$ is constructed from $F(\cdot\sc\lambda,\lambda')$, we construct $\hat R^-_{\psi,n}(\hat\lambda_{t-1},\lambda',\delta')$, $\hat R_{\psi,n}(\hat\lambda_{t-1},\lambda')$, and $\hat R^+_{\psi,n}(\hat\lambda_{t-1},\lambda',\delta')$ from $\hat F_{U,n}(\cdot \mathbin{;} \hat\lambda_{t-1},\lambda',\delta')$, $\hat F_{n}(\cdot \mathbin{;} \hat\lambda_{t-1},\lambda',\delta')$, and $\hat F_{L,n}(\cdot \mathbin{;} \hat\lambda_{t-1},\lambda',\delta')$, respectively. Note that the risk measure's lower bound $\hat R^-_{\psi,n}(\hat\lambda_{t-1},\lambda',\delta')$ corresponds to the CDF upper bound $\hat F_{U,n}(\cdot \mathbin{;} \hat\lambda_{t-1},\lambda',\delta')$ and vice versa (see App.~\ref{app:quantile-ext-proofs}). The bounds on risk measure $R_\psi(\hat\lambda_{t-1},\lambda')$ satisfy the following:

$$
\mathbb{P}\left(\hat R^-_{\psi,n}\psi(\hat\lambda_{t-1},\lambda',\delta') \leq R_\psi(\hat\lambda_{t-1},\lambda') \leq \hat R^+_{\psi,n}(\hat\lambda_{t-1},\lambda',\delta') \right) \geq 1-\delta'
$$

Let $c(n,\delta',\lambda')=\max\{ 
\hat R^+_{\psi,n}(\hat\lambda_{t-1},\lambda',\delta')
-
\hat R_{\psi,n}(\hat\lambda_{t-1},\lambda'),
\hat R_{\psi,n}(\hat\lambda_{t-1},\lambda')
-
\hat R^+_{\psi,n}(\hat\lambda_{t-1},\lambda',\delta')
\}$. Finally, we compute the confidence width $c(n,\delta')$ using the technique described in App.~\ref{app:bounds}. We are now ready to extend Thm.~\ref{thm:risk-control} to quantile-based risk measures.

\begin{theorem}[Quantile risk control of PRC]\label{thm:quantile}
In Alg.~\ref{alg:main}, replace the confidence width with the one derived above. Further, replace $\hat R_n(\hat\lambda_{t-1},\lambda)$ with $\hat R_{\psi,n}(\hat\lambda_{t-1},\lambda)$ in the definition of $V$ in line~\ref{alg:line-sett}. Let $u,v\in[1,\infty]$ with $1/u+1/v=1$. If the following four conditions hold: (i) the distribution mapping is $(\gamma,u,\ell(\cdot,\lambda))$-sensitive for all $\lambda\in\Lambda$, (ii) the loss function $\ell(z,\lambda)$ is continuous in $\lambda$ for all $z$, (iii) $\tau \geq \gamma \left[ \int_0^1 |\psi(p)|^vdp \right]^{1/v}$, and (iv) the initial joint solve of $(\tilde T,\Delta\lambda)$ produces at least one value; then, this modified version of Alg.~\ref{alg:main} guarantees that with probability $1-\delta$, the quantile risk $R_\psi$ of the iterates $\hat\lambda$ satisfies the same three guarantees as those in Thm.~\ref{thm:risk-control}.
\end{theorem}

\section{Experiment}\label{sec:exp}

This section considers the application of our framework to credit scoring in both the expected and quantile risk settings. We present the most important results here and leave details to App.\ref{app:exp}.

\begin{figure}
   \centering
    \begin{subfigure}[b]{0.5\textwidth}
        \includegraphics[width=\textwidth]{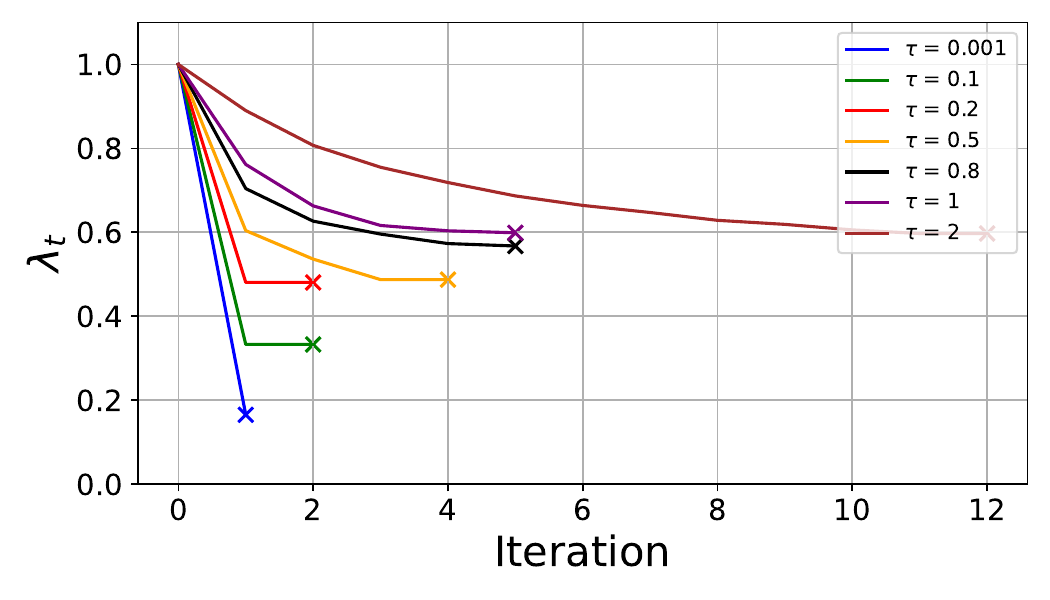}
        \label{fig:cser-a}
    \end{subfigure}%
    \hfill
    \begin{subfigure}[b]{0.5\textwidth}
        \includegraphics[width=\textwidth]{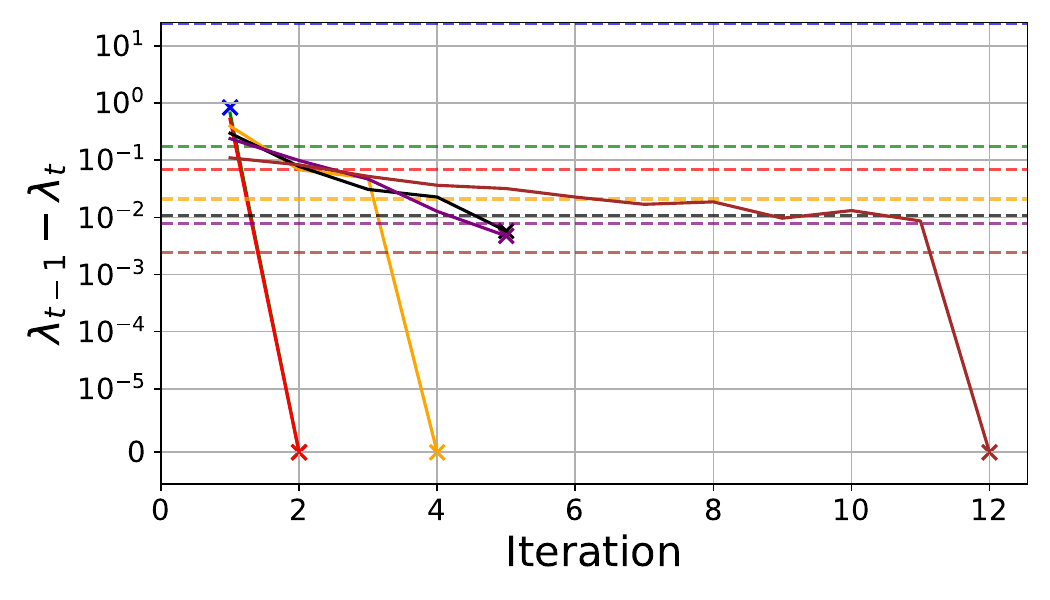}
        \label{fig:cser-b}
    \end{subfigure}%

    \begin{subfigure}[b]{0.5\textwidth}
        \includegraphics[width=\textwidth]{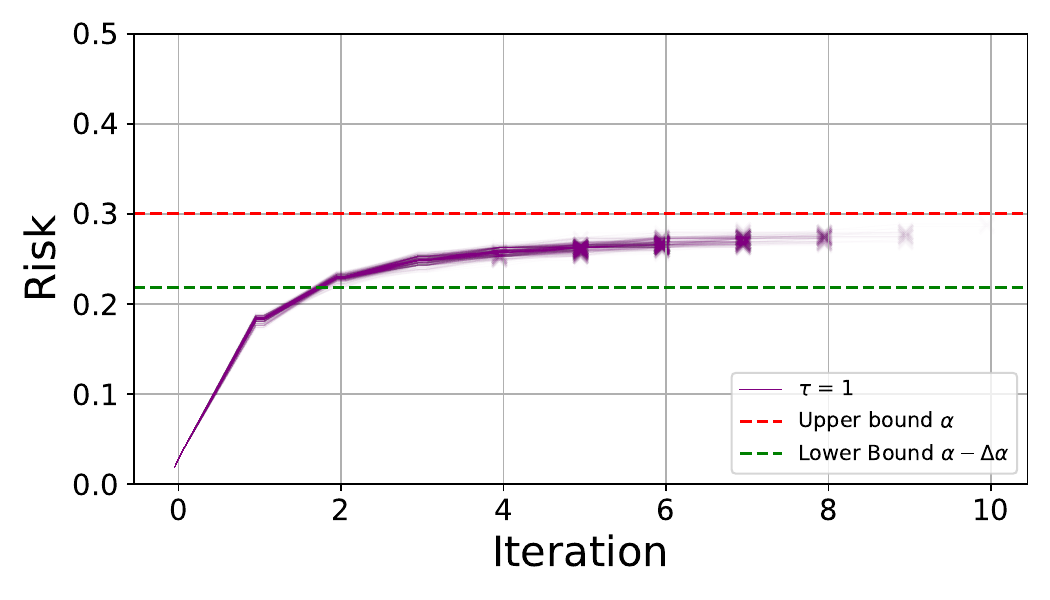}
        \label{fig:cser-c}
    \end{subfigure}%
    \hfill
    \begin{subfigure}[b]{0.5\textwidth}
        \includegraphics[width=\textwidth]{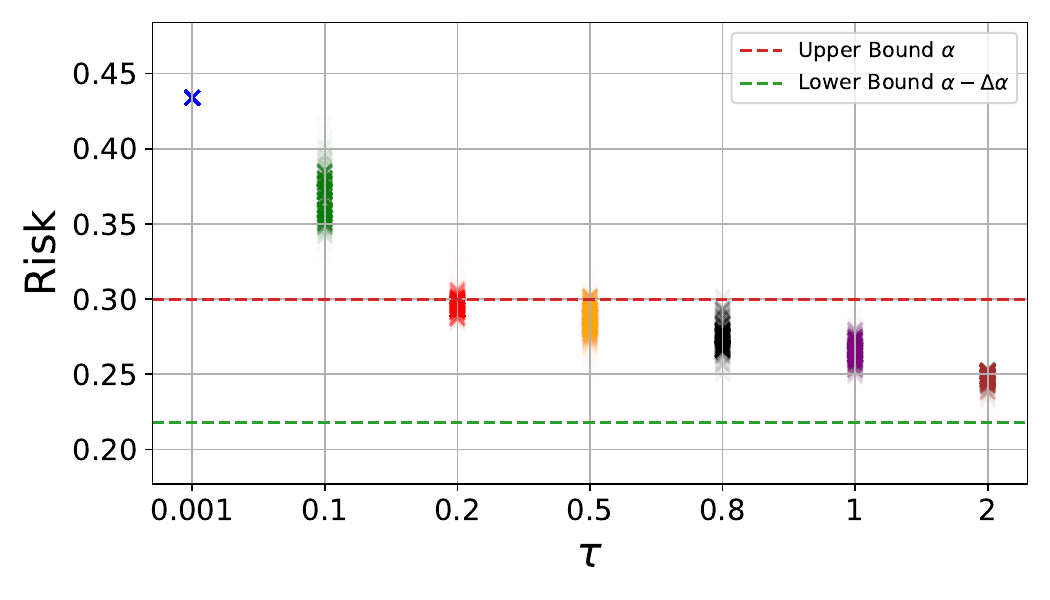}
        \label{fig:cser-d}
    \end{subfigure}
    % \caption{\victor{TODO: insert new plots} Type II error control in credit scoring. \textbf{Top left:} $1,000$ trajectories of $\lambda_t$ for different values of $\tau$, each trajectory corresponding to a different cut between calibration and validation sets. As $\tau$ increases, the decision threshold becomes more conservative, staying close to around $0.5$ for $\tau\geq0.5$. \textbf{Top right:} $3$ samples trajectories of the risk.  The iterates $\lambda_t$ were chosen based on the calibration set, and the displayed risk is on the validation set. Let $\{z_t^{(i)}\sim\cD(\lambda_t)\}_{i=1}^{n_v}$ denote the validation samples from iteration $t$. For each iteration $t$, the plot displays $\frac{1}{n_v}\sum_{i=1}^{n_v} \ell(z_t,\lambda_t)$ followed by $\frac{1}{n_v}\sum_{i=1}^{n_v} \ell(z_t,\lambda_{t+1})$, most noticable when $t=0$. Risk is controlled for both quantities for high enough $\tau$. \textbf{Bottom left:}  A scatterplot of $R(\hat\lambda_T)$ (evaluated on the validation set) vs. $\tau$. $\tau$ needs to be sufficiently high in order for the risk to be controlled. \textbf{Bottom right:} The relative proportion of trajectories for each $\tau$ in which the selected $\hat\lambda_T$ is so that $\alpha-\Delta\alpha\leq R(\hat\lambda_T)\leq\alpha$.}
    \vspace{-1.5em}
    \caption{Type II error control in credit scoring. We use the same coloring scheme to label trajectories with different settings of $\tau$ across the subplots. \textbf{Top left}: As $\tau$ increases, the incremental updates to $\hat\lambda_t$ become more fine-grained and conservative. This safety comes at the trade-off of longer trajectories. \textbf{Top right}: An illustration of our stopping criteria for different values of $\tau$. To guarantee tightness with higher $\tau$, $\Delta\lambda$--displayed as the colored, horizontal dotted lines--needs to be lower. Once $\hat\lambda_{t-1} - \hat\lambda_t$ falls below $\Delta\lambda$, PRC ends exploration and returns $\hat\lambda_T=\lambda_t$. \textbf{Bottom left}: $1,000$ trajectories with $\tau=1$, each trajectory corresponding to a different cut between calibration and validation sets. The iterates $\hat\lambda_t$ were chosen based on the calibration set, and the displayed risk is on the validation set. Let $\{z_{t,i}\sim\cD(\hat\lambda_t)\}_{i=1}^{n_v}$ denote the validation samples from iteration $t$. For each iteration $t$, the plot displays $\frac{1}{n_v}\sum_{i=1}^{n_v} \ell(z_{t,i},\hat\lambda_t)$ followed by $\frac{1}{n_v}\sum_{i=1}^{n_v} \ell(z_{t,i},\hat\lambda_{t+1})$, most noticeable when $t=0$. Risk is controlled for both quantities. \textbf{Bottom right}: A scatterplot of $R(\hat\lambda_T)$ for each $\tau$. For $\tau\geq\gamma$, PRC achieves its goal of being $(\alpha,\delta)$-performative-risk-controlled. Even for some values of $\tau<\gamma$ (e.g., $\tau=0.5,0.8,1.0$), the final risk is still controlled (also see App.~\ref{app:cs-ev}).}
    % \caption{Type II error control in credit scoring. \victor{TODO: update} \textbf{Top left}: As $\gamma$ increases, the decision threshold becomes more conservative. At $\gamma=10.0$, the performative shift becomes so severe that it is unsafe for the model to be used at all. \textbf{Top right}: An illustration of our stopping criteria. Once $\lambda_{t-1} - \lambda_t$ falls below the threshold $\Delta\lambda$, we end exploration and return $\hat\lambda=\lambda_t$. \textbf{Bottom left}: $1,000$ trajectories with $\tau=\gamma=1.0$, each trajectory corresponding to a different cut between calibration and validation sets. The iterates $\lambda_t$ were chosen based on the calibration set, and the displayed risk is on the validation set. Let $\{z_t^{(i)}\sim\cD(\lambda_t)\}_{i=1}^{n_v}$ denote the validation samples from iteration $t$. For each iteration $t$, the plot displays $\frac{1}{n_v}\sum_{i=1}^{n_v} \ell(z_t,\lambda_t)$ followed by $\frac{1}{n_v}\sum_{i=1}^{n_v} \ell(z_t,\lambda_{t+1})$, most noticable when $t=0$. Risk is controlled for both quantities. \textbf{Bottom right}: A scatterplot of $R(\hat\lambda)$ (evaluated on the validation set) vs. the total number of iterations in the trajectory, for $1,000$ trajectories. Of these $1,000$ final risks, $990$ falls between the bounds (recall $\delta=0.1$). PRC achieves its goal of being $(\alpha,\delta)$-performative-risk-controlled. We observe some looseness because the safety parameter $\tau$ is set generously to ensure that performative shift is counteracted.}
    \label{fig:credit-scoring-expected-risk}
\end{figure}

\textbf{Type II Error Control for Credit Scoring.} We first consider type II error control for the binary classification task. Labeling the positive class as $y=1$, we form the standard prediction $\cT_{\lambda,\epsilon}(x)=\text{clip}(\frac{1-\lambda+\epsilon-f(x)}{2\epsilon},0,1)$\footnote{We use $\epsilon=10^{-4}.$ This standard prediction approximates the indicator function $\bm{1}\{f(x)\leq1-\lambda\}$ as $\epsilon\rightarrow0$. For simplicity, our theory assumes a continuous loss function and does not deal with losses with bounded jump discontinuities, hence the approximation here. However, our theory can easily be extended to handle these cases as well, as done in \citet{angelopoulos2022conformal}.}, where $\lambda\in[\lambda_\text{min},\lambda_\text{safe}]=[0,1]$. Here, $f(x)$ estimates the probability that $x$ is positive. We assign $x$ to the negative class when $\cT_{\lambda,\epsilon}(x)=0$ and abstain from making a prediction otherwise. We consider the type II error $\ell(y,\cT_{\lambda,\epsilon}(x))=y(1-\cT_{\lambda,\epsilon}(x))$, accumulating loss for each positive instance assigned to the negative class. We apply this formulation to automatic credit approval. Credit applicants submit their applications with features $x$. The model $f(x)$ predicts the probability of a serious delinquency, and $y$ labels whether one occurred.  When $\cT_{\lambda,\epsilon}(x)=0$, the application is automatically accepted; otherwise, it is flagged and subjected to further human review. The threshold $\lambda\in[0, 1]$ trades off the balance of these two, with lower $\lambda$ corresponding to more automatic acceptances. We focus solely on the Type II errors made by the model.

However, $\lambda$ is \textit{performative}. We are not deploying $\lambda$ into a static distribution; applicants respond by maximizing their chances of receiving credit approval. Denote $x(\lambda)$ as an applicant's features in response to the deployment of $\lambda$. We simulate this shift \textit{in the features} by impacting $f(x(\lambda))$ as follows:
$$f(x(\lambda))=\begin{cases}
\max(0,f(x)-s) & \text{if } f(x)-s\leq 1-\lambda \\
f(x) & \text{otherwise}
\end{cases}$$

where $s(=0.3)$ is a value we choose. This shift is $(\gamma,1,\ell(\cdot,\lambda))$-sensitive for all $\lambda$ and choice of $s$; we explore the validity of this claim in App.~\ref{app:cs-ev}.

We demonstrate results on a class-balanced subset of a Kaggle credit scoring dataset \cite{GiveMeSomeCredit}, following \cite{perdomo2020performative}. Our simulation experiment is reported in Figure \ref{fig:credit-scoring-expected-risk}. We control risk at $\alpha=0.3$ with $\Delta\alpha=0.082$. We use $n=2000$, the CLT bound, and $\delta=0.1$. We experimentally verify that $\gamma\leq1.38$ (see App.~\ref{app:cs-ev}). As expected from Thm.~\ref{thm:risk-control}, PRC iterates safely as it converges upon a final $\hat\lambda$ that is both risk-controlled and not too conservative for large enough $\tau$.

\textbf{Quantile Risk Control in Credit Scoring.} In a highly imbalanced dataset where the negative class is much more common than the positive class, risk control on a small expected type II error will not be very meaningful. Quantile-based risk measures target scenarios like this one, in which we seek guarantees on the tail end of negative outcomes. Given that we are investigating quantile-based risk measures, we adjust our setup to more accurately reflect a setting in which they would be useful. For each positive delinquency on each iteration, we simulate a realized cost by drawing from $U[0,1]$, which can be thought of as the amount the institution loses on an individual, via being late on payments or defaulting, scaled between $0$ and $1$. We also assume knowledge of the base rate $p$ of a credited individual experiencing a serious delinquency; in this experiment, $p=0.0640$.

We focus on the $\beta$-CVaR risk measure, with $\beta=0.9$. We use a quantile-based-CLT confidence width that depends on $n$, $\delta$, $\tilde T$, $\beta$, and $p$ (see App.~\ref{app:cs-cvar}). We increase $n$ to $10,000$. Note that we require this many samples because $1-p$ of the population has a loss of $0$ and is not informative for us in setting $\lambda$. This setting follows Thm~\ref{thm:quantile} with $u=1$, $v=\infty$, and $\gamma \left[ \int_0^1 |\psi(p)|^vdp \right]^{1/v} = \gamma/(1-\beta)\leq\tau$. We experimentally verify $\gamma\leq 0.205$, so Thm.~\ref{thm:quantile} applies when $\tau\geq 2.05$. Fig. \ref{fig:credit-scoring-cvar} shows our results for $\alpha=0.25$ and $\Delta\alpha=0.12$. We observe that the confidence bounds are somewhat generous; each trajectory ends with a final $R(\hat\lambda_T)$ inside the bounds. However, this width is necessary given that this dataset is so highly skewed and that we are controlling the most variable $10\%$ of losses.

\begin{figure}[!ht]
    \centering
    \begin{subfigure}[b]{0.5\textwidth}
        \includegraphics[width=\textwidth]{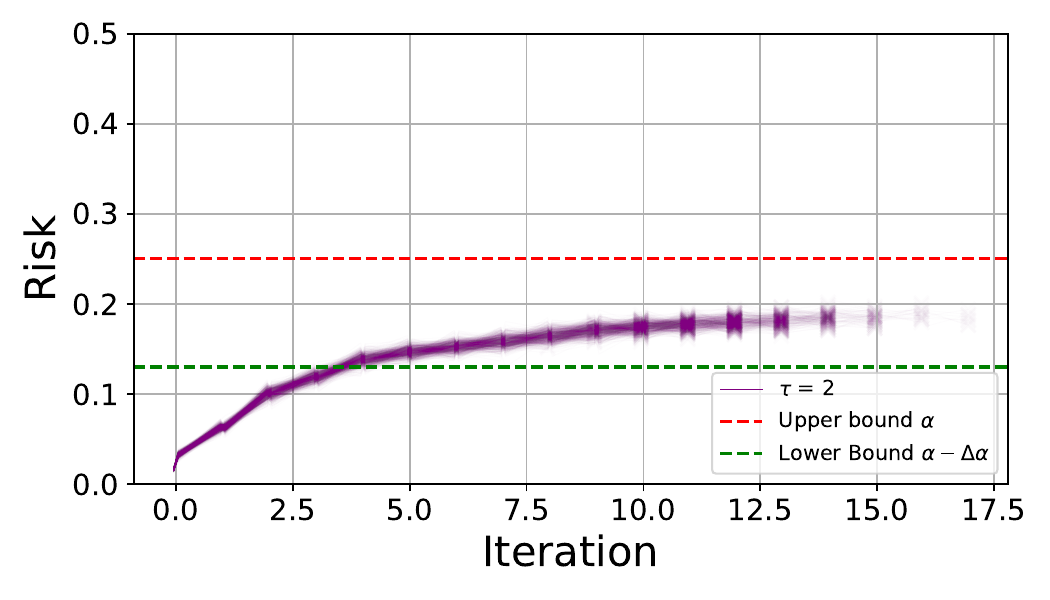}
        \label{fig:cscvar-c}
    \end{subfigure}%
    \hfill
    \begin{subfigure}[b]{0.5\textwidth}
        \includegraphics[width=\textwidth]{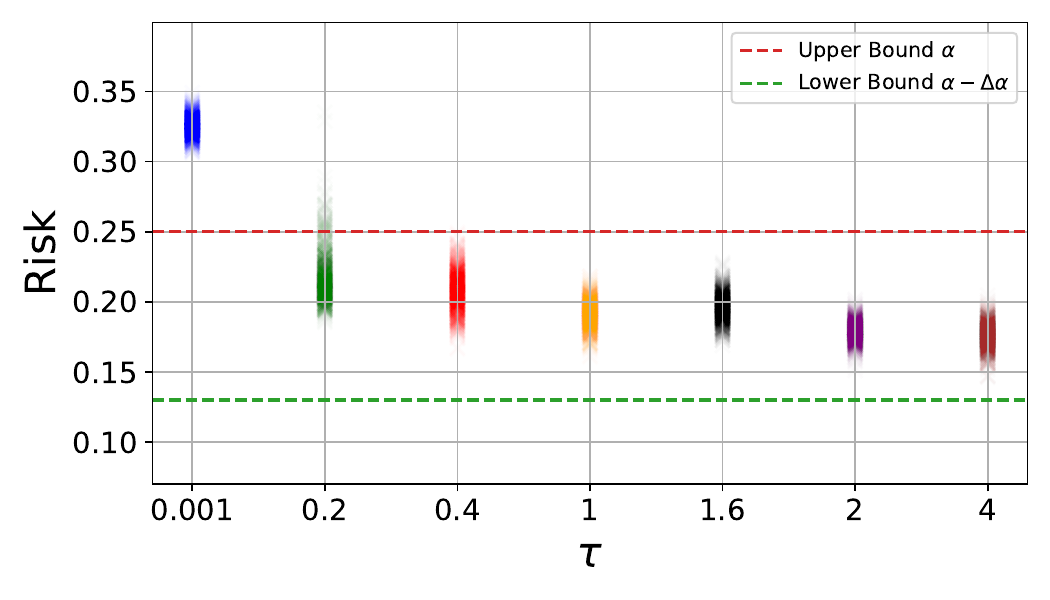}
        \label{fig:cscvar-d}
    \end{subfigure}
    \caption{$90\%$-CVaR type II error control in credit scoring. These plots correspond to the bottom two plots in Fig.\ref{fig:credit-scoring-expected-risk} but for the $90\%$-CVaR risk measure. The left plot shows $\tau=2$ instead of $\tau=1$.}
    \label{fig:credit-scoring-cvar}
\end{figure}

\section{Conclusion}\label{sec:outro}

We have introduced a theoretically sound and empirically effective framework for risk control under performativity, a previously unexplored topic in the risk control literature. To mitigate performativity, our framework iteratively refines a series of decision thresholds, careful to control risk. As a concrete example, we examined the deployment of a credit scoring system for loan approvals and found that PRC successfully controls risk, even as credit applicants seek to strategically manipulate their outcome. One limitation of our current framework is inherited from the line of work on risk control. Since risk control focuses on calibrating models via a simple postprocessing, it may not always be as effective as finetuning the entire model. For instance, this simple postprocessing may sometimes only output $\lambda_{\text{safe}}$. Nevertheless, given the growing role of blackbox models in shaping policies of societal significance, PRC offers an effective and proactive strategy for their safe deployment.

\bibliographystyle{apalike}
\bibliography{refs}

%%%%%%%%%%%%%%%%%%%%%%%%%%%%%%%%%%%%%%%%%%%%%%%%%%%%%%%%%%%%

\newpage

\appendix
\section{Proofs}\label{app:proofs}

\subsection{Connection to sensitivity in performative prediction}
\label{app:sensitivity}

In \citet{perdomo2020performative}, $\epsilon$-sensitivity of the distribution mapping $\cD(\cdot)$ is a key assumption. The parameter $\epsilon$ controls the magnitude of the distribution shift, as measured on the samples $z=(x,y)$. In our formulation, we view shifts in $z$ as affecting some downstream loss function $\ell$. Hence, bounds on the Wasserstein-1 distance between two distributions of $z$'s can be translated to bounds on the Wasserstein-1 distance between two distributions of the loss, for an appropriately chosen $\tilde\gamma$. Thus, the original $\epsilon$-sensitivity assumption can be seen as a special case of our assumption.

\subsection{Choice of bounds}
\label{app:bounds}

The purpose of the confidence width $c(n,\delta')$ is to guard against generalization error. Because we need access to $c(n,\delta')$ before the initial deployment and $c(n,\delta')$ must guard against worst-case distributions that could arise from performative shift, the bounds presented here are looser than those found in the literature, yet still fairly tight. In this section, we detail how we can derive confidence widths from a variety of concentration inequalities, a comparison of which is shown in Fig.~\ref{fig:bounds.pdf}, for a bounded loss scaled in the range $[0,1]$.

\paragraph{Hoeffding's inequality}

For any $\lambda,\lambda'\in\Lambda$, if $z\sim\cD(\lambda)$, then $\ell(z,\lambda')\in[0,1]$. Because these are bounded random variables, Hoeffding's inequality applies, which says that for any $\delta'\in(0,1)$,
$$
\Prob{ \left|\hat R_n(\lambda,\lambda') - R(\lambda,\lambda') \right| \leq c(n,\delta') } \geq 1-\delta'
$$
where $c(n,\delta')=\sqrt{\frac{1}{2n} \ln(2/\delta')}$. While we present Hoeffding's inequality for illustration, we use sharper bounds in practice. The rest of the bounds we consider in this section are sample-dependent. We use knowledge of $\alpha$ to calculate their confidence widths, as detailed next.
\begin{figure}
	\centering
	\includegraphics[width=0.6\linewidth]{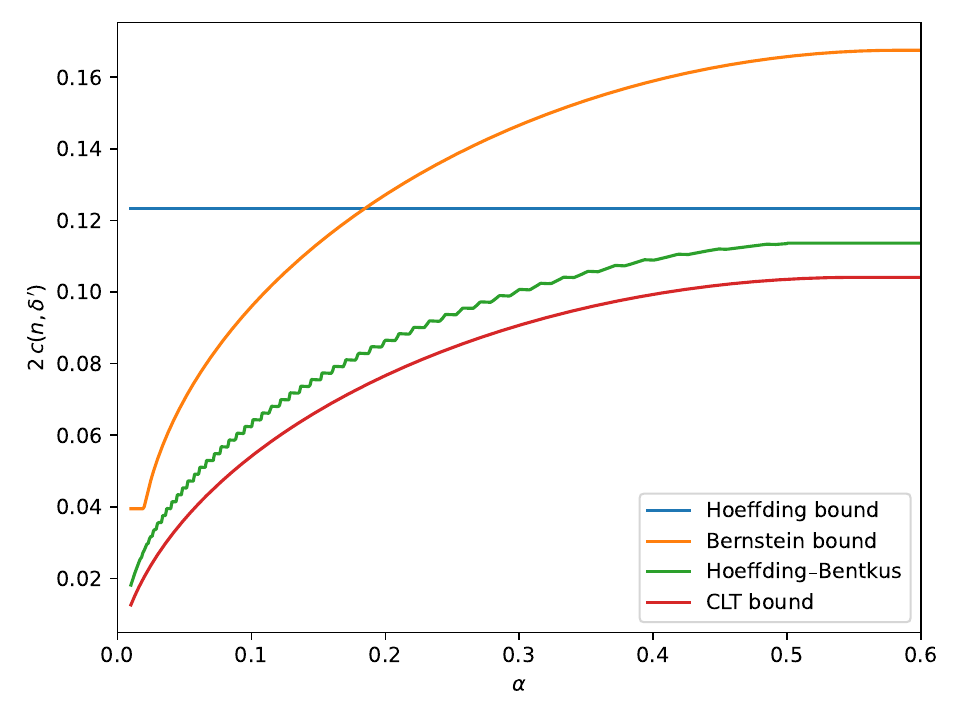}
	\caption{A comparison of Hoeffding and the risk level-dependent bounds. Plots $2c(n,\delta')$ vs. $\alpha$ for $n=2000$, $\tilde T=100$, $\delta=0.1$, and $\delta'=\delta/\tilde T$. $2c(n,\delta')$ represents the lower limit of achievable tightness $\Delta\alpha$. Lower confidence widths are better.}
	\label{fig:bounds.pdf}
\end{figure}

\paragraph{Risk level-dependent bounds}\label{app:risk-level-dependent-bounds} Sample-dependent concentration bounds require some care. Our algorithm calculates the confidence width $c(n,\delta')$ \textit{before the first deployment} without access to the distributions of future deployments. This task of using sample-dependent bounds without initial access to them is further complicated by performative shift, which impedes us from more specific characterizations of these future distributions, and hence from obtaining sharper bounds. However, we have knowledge of the risk control level $\alpha$ and can use that information to minimize $c(n,\delta')$ while making it wide enough to control risk throughout, and at the end of, the iterative process.

Specifically, on iteration $t$, we receive samples from the previous round $\{ z_{t-1,i}\}_{i=1}^n\overset{\text{i.i.d.}}{\sim} \mathcal{D}(\hat\lambda_{t-1})$ and can calculate $\hat R_n(\hat\lambda_{t-1},\lambda)$ for any $\lambda$. $\hat R_n(\hat\lambda_{t-1},\lambda)$ is the sample mean loss when evaluating with $\lambda$; from it, we can calculate the maximum sample variance, which occurs when the loss distribution is Bernoulli with probability $\hat R_n(\hat\lambda_{t-1},\lambda)$. Since our concentration inequalities use these two quantities to construct their bounds, the exact distribution we draw from does not matter, and the derived confidence widths will work for the distribution \textit{at any round}, regardless of performativity.

Next, note that our algorithm is functionally equivalent if we replace Line \ref{alg:line-sett} of Alg. \ref{alg:main} with the following update rule:

$$\lambda_t^{\text{mid}}\gets \inf \{ \lambda\in\Lambda \mid \hat R^+_n(\hat\lambda_{t-1},\lambda,\delta')+\tau\max(0,\hat\lambda_{t-1}-\lambda)< \alpha \}.$$

Since $\tau\max(0,\hat\lambda_{t-1}-\lambda)\geq0$ always, the candidates $\lambda$ for $\lambda_t^{\text{mid}}$ are a subset of $S_\alpha:=\{ \lambda\in\Lambda \mid \hat R^+(\hat\lambda_{t-1},\lambda,\delta')< \alpha \}$. Assume for all $\lambda$ we have access to a \textit{pointwise} confidence width $c(n,\delta',\lambda)$ that bounds $R(\hat\lambda_{t-1},\lambda)$ by knowledge of the sample statistics calculated from $\hat R_n(\hat\lambda_{t-1},\lambda)$ above. Let $\hat R_n^\pm(\hat\lambda_{t-1},\lambda,\delta') = \hat R_n(\hat\lambda_{t-1},\lambda,\delta') \pm \sup_{\lambda'\geq\lambda} c(n,\delta',\lambda')$. Define the update that does not consider performativity as $\tilde \lambda_t^\text{mid}:=\inf\{S_\alpha \}$. We set the confidence width as 
$$
c(n,\delta'):=\max \{ c(n,\delta',\lambda') \mid \lambda'\geq \tilde\lambda_t^\text{mid} \}.
$$

This setting ensures that throughout the entire trajectory of $\hat\lambda_t$ for which we bound $\hat R_n(\hat\lambda_{t-1}\hat\lambda_t) +\tau (\hat\lambda_{t-1}-\hat\lambda_t) \leq \alpha - c(n,\delta')$, we can guarantee with high probability that $R(\hat\lambda_{t-1},\hat\lambda_t) +\tau (\hat\lambda_{t-1}-\hat\lambda_t) \leq \alpha$ also. Next, we use this general result to obtain confidence widths from Bernstein's inequality, the Hoeffding-Benkus inequality, and the central limit theorem.

\paragraph{Bernstein's inequality}

\begin{proposition}[Empirical Bernstein Bounds {\cite[Theorem 11]{maurer2009empiricalbernsteinboundssample}}] Let $\mathbf{X}=(X_1,...,X_n)$ be a vector of independent random variables with values in $[0,1]$. Let $\delta$ > 0. Then with probability at least $1-\delta$, we have

$$
\mathbb{E} \left[ P_n(\mathbf{X}) \right] \leq P_n(\mathbf{X}) + \sqrt{\frac{2 V_n(\mathbf{X}) \ln 2/\delta}{n}} + \frac{7 \ln 2/\delta}{3 (n - 1)}
$$

where $P_n(\mathbf{X})$ denotes the sample mean and $V_n(\mathbf{X})$ denotes the sample variance.
\end{proposition}

We observe the sample mean $\hat R_n(\hat\lambda_{t-1},\lambda)$ and can upper bound the sample variance with $\hat R_n(\hat\lambda_{t-1},\lambda) (1 - \hat R_n(\hat\lambda_{t-1},\lambda))$, which is obtained when the losses are Bernoulli. We further convert the one-sided bound into a two-sided bound, giving us the following:
$$
c(n,\delta',\lambda)=\sqrt{ \frac{2\hat R_n(\hat\lambda_{t-1},\lambda) (1 - \hat R_n(\hat\lambda_{t-1},\lambda)) \ln(4/\delta')}{n} } + \frac{7\ln(4/\delta')}{3(n-1)}.
$$
We can then use the technique described above in App. \ref{app:risk-level-dependent-bounds} to obtain the confidence width.

\paragraph{The Hoeffding-Bentkus inequality}

\begin{proposition}[Hoeffding-Bentkus inequality p-values {\cite[Proposition 1]{angelopoulos2022ltt}}]
\textit{The following is a valid p-value for $\mathcal{H}: R(\lambda)>\beta$ for $n$ samples}
\begin{equation}
    p^{\mathrm{HB}}(n,\hat R_n(\lambda), \beta) = \min \left( \exp \left\{-n h_1\left(\hat R_n(\lambda) \wedge \beta, \beta\right) \right\},\ e \mathbb{P}\left(\mathrm{Bin}(n, \beta) \le \left\lceil n \hat R_n(\lambda) \right\rceil \right) \right),
\end{equation}
where $R(\lambda)$ is a bounded risk, $\hat R_n(\lambda)$ is the sample risk, $h_1(a, b) = a \log(a/b) + (1 - a) \log((1 - a)/(1 - b))$ and $\wedge$ denotes the minimum.
\end{proposition}

$R(\hat\lambda_{t-1},\lambda)$ falls in the interval $$( \hat R_n^-(\hat\lambda_{t-1},\lambda,\delta'), \hat R_n^+(\hat\lambda_{t-1},\lambda,\delta') ):= (\hat R_n(\hat\lambda_{t-1},\lambda,\delta')-c, \hat R_n(\hat\lambda_{t-1},\lambda,\delta')+c)$$ if the following two hypotheses are rejected:
\[
\begin{aligned}
\mathcal{H}^{UB}: &\;
R(\hat\lambda_{t-1},\lambda)
>
\beta^{UB}
:=
\hat R_n^+(\hat\lambda_{t-1},\lambda,\delta')
\quad\text{with p-value}\;
p^{UB}(n,\hat R_n(\hat\lambda_{t-1}, \lambda),\beta^{UB}),
\\[6pt]
\mathcal{H}^{LB}: &\;
1 - R(\hat\lambda_{t-1},\lambda)
>
1 - \beta^{LB}
:=
1 - \hat R_n^-(\hat\lambda_{t-1},\lambda,\delta')
\quad\text{with p-value}\;
p^{LB}(n,1-\hat R_n(\hat\lambda_{t-1}, \lambda),1-\beta^{LB}).
\end{aligned}
\]

We simply choose the narrowest confidence width that provides us with $\delta'$ error control:
$$
c(n,\delta',\lambda):= \arg \min_{c\geq 0} p^{UB}(n, \hat R_n(\hat\lambda_{t-1},\lambda),\beta^{UB}) + p^{LB}(n,1-\hat R_n(1-\hat\lambda_{t-1},\lambda),1-\beta^{LB}) \leq \delta'.
$$

\paragraph{Central limit theorem}

Given $\hat R_n(\lambda_{t-1},\lambda)$, the maximum sample variance is $$\hat\sigma^2(n,\lambda)=\frac{1}{n-1}\sum_{i=1}^n \left( \ell(z^{(i)}_{t-1},\lambda)- \hat R_n(\lambda_{t-1},\lambda)\right)^2.$$

Denote the cumulative distribution function (CDF) of the standard normal as $\Phi$. With probability $1-\delta'$, $| R(\lambda_{t-1},\lambda) - \hat R_n(\lambda_{t-1},\lambda) \leq c(n,\delta',\lambda)$ where $c(n,\delta',\lambda)=\Phi^{-1}(1-\delta'/2) \hat\sigma(\lambda)/\sqrt{n}$.

\subsection{Asymptotic results for the lower bound}\label{app:asymptotics}

Recall that $\Delta\alpha=2\tau\Delta\lambda+2c(n,\delta/\tilde T)$ where $\tilde T\geq\lceil (\lambda_\text{safe}-\lambda_\text{min})/\Delta\lambda\rceil$. First, we need to ensure that there exist solutions $(\tilde T,\Delta\lambda)$ in this asymptotic regime. For a given $\Delta\lambda$, we can choose the minimum $\tilde T$ subject to its constraint and let $n\rightarrow \infty$. Doing so allows $c(n,\delta/\tilde T)\rightarrow0$. Since $\Delta\lambda$ can be chosen arbitrarily small, we can make $\Delta\alpha$ go to zero.

Suppose we let $\Delta\lambda\rightarrow 0$ and scale the number of samples $n$ as $\Delta\lambda=Cn^{-r}$ (for some constant $C$ and sufficiently high $r\geq 1/2$). We remark that the confidence width for the bounds we consider can be expressed asymptotically as $c(n,\delta/\tilde T)=O(\sqrt{\frac{\ln(\tilde T/\delta)}{n}})=O(\sqrt{\frac{\ln(1/\Delta\lambda)}{n}})=O(\frac{r\ln(n)}{\sqrt{n}})=O(\frac{\ln n}{\sqrt{n}})$. Hence, asymptotically

\begin{align*}
\Delta\alpha &= 2\tau\Delta\lambda + 2c(n,\delta/\tilde T) \\
&= O(n^{-r})+O(\frac{\ln n}{\sqrt{n}}) \\
&= O(\frac{\ln n}{\sqrt{n}})
\end{align*}
for $r\geq 1/2$. PRC allows users to choose a vanishingly tight lower bound $\Delta\alpha\rightarrow0$.

\subsection{Main proofs}\label{app:main-proofs}

In this section, we build towards our final goal of proving Thm.~\ref{thm:risk-control}. We do so by first showing that the performative error, arising from risk measures with the same evaluation threshold but whose samples are drawn from different distributions, can be bounded. We then demonstrate the validity of UCB calibration to bound the generalization error. Taken together, we can then prove Thm.~\ref{thm:risk-control}.

The following lemma uses the definition of a quantile-based risk measure from Def.~\ref{def:quantile-risk-measure}. Notably, $\psi(p)$ is a weighting function; $\psi(p)=1$ when we consider expected risk. It is important to remark that when $\psi(p)=1$, $\left[ \int_0^1 |\psi(p)|^udp \right]^{1/u}=1$ for $u\in[1,\infty]$.

\begin{lemma}
\label{lemma:sensitivity-quantile}
Let $u,v\in[1,\infty]$ with $1/u+1/v=1$. If $\left[ \int_0^1 |\psi(p)|^udp \right]^{1/u}\leq M$ and $\mathcal{D}(\cdot)$ is $(\gamma,v,\ell(\cdot,\lambda))$-sensitive for all $\lambda\in\Lambda$, then for any $\lambda_1,\lambda_2,\lambda\in\Lambda$:
$$
\left|
R_\psi(\lambda_1,\lambda) - R_\psi(\lambda_2,\lambda)
\right|
 \leq \gamma M |\lambda_1 - \lambda_2|.
$$
\end{lemma}

\begin{proof}%[Proof of Lemma \ref{lemma:sensitivity-quantile}]
\begin{align*}
|R_\psi(\lambda_1,\lambda)-R_\psi(\lambda_2,\lambda)|
&=
\left|
\int_0^1 \psi(p) F^{-1}(\cdot\mathbin{;} \lambda_1,\lambda) dp - 
\int_0^1 \psi(p) F^{-1}(\cdot\mathbin{;} \lambda_2,\lambda) dp
\right|\\
&=
\left|
\int_0^1 \psi(p) \left[ F^{-1}(\cdot\mathbin{;} \lambda_1,\lambda) - F^{-1}(\cdot\mathbin{;} \lambda_2,\lambda) \right] dp 
\right|\\
&=
\int_0^1 \psi(p) \left| F^{-1}(\cdot\mathbin{;} \lambda_1,\lambda) - F^{-1}(\cdot\mathbin{;} \lambda_2,\lambda) \right| dp \\
&= \left( \int_0^1 |\psi(p)|^u dp \right)^{1/u} \left( \int_0^1 \left| F^{-1}(p\mathbin{;} \lambda_1,\lambda) - F^{-1}(p\mathbin{;} \lambda_2,\lambda) \right|^v dp \right)^{1/v}\\
&\leq M \left( \int_0^1 \left| F^{-1}(p\mathbin{;} \lambda_1,\lambda) - F^{-1}(p\mathbin{;} \lambda_2,\lambda) \right|^v dp \right)^{1/v}\\
&= MW_v(\mathcal{D}(\lambda_1),\mathcal{D}(\lambda_2))\\
&\leq M\gamma|\lambda_1-\lambda_2|
\end{align*}

\end{proof}

As a note on notation, from now on until the rest of this proof section, we write the risk measure as $R$, understanding that this refers to the more general form $R_\psi$ (where the subscript $\psi$ has been dropped).

In the following proposition, we restate Theorem~A.1 in~\citet{bates2021distribution} in our performative setting. This theorem is key to turning the pointwise bound used in line \ref{alg:line-sett} of Algorithm \ref{alg:main} into a proof of the validity of each iterate $\hat\lambda$.

\begin{proposition}[Validity of UCB Calibration]
\label{prop:prc-core}
Assume we have deployed some $\lambda \in \Lambda$ to obtain samples from the distribution $\mathcal{D}(\lambda)$. Let the true risk function $R(\lambda, \cdot): \Lambda \to \mathbb{R}$ and its $1 - \delta'$ upper confidence bound (UCB) $\hat R_n^+(\lambda, \cdot; \delta') : \Lambda \to \mathbb{R}$ be continuous and non-increasing functions. Suppose there exists some $\lambda' \in \Lambda$ such that both
\[
R(\lambda, \lambda') \leq \alpha \quad \text{and} \quad \hat R_n^+(\lambda, \lambda'; \delta') \leq \alpha.
\]
Consider the smallest $\lambdasp$ whose UCB falls below the risk level:

\[
\lambdasp:=\inf \{\lambda'\in\Lambda: \hat R_n^+(\lambda,\lambda',\delta')< \alpha \}
\]

Then,
\[
\mathbb{P}\left( R(\lambda,\lambdasp)\leq\alpha \right) \geq 1-\delta'
\]
\end{proposition}

\begin{proof}
First, we claim that $\hat R_n^+(\lambda,\cdot,\delta'):=\hat R_n(\lambda,\cdot)+c(n,\delta')$ is non-increasing and continuous. By the property of the loss function, $\hat R_n(\lambda,\cdot)$ is non-increasing and continuous. Further, $c(n,\delta')$ is constant by construction, proving the claim.

Assume $R(\lambda,\lambdasp)>\alpha$. Consider the smallest $\lambda_*$  that controls the risk
\[
\quad
\lambda_*:=\inf \{\lambda'\in\Lambda: R(\lambda,\lambda')\leq \alpha \}
\]

Because $R(\lambda,\cdot)$ is non-increasing and continuous, $\lambdasp<\lambda_*$. Further, since $\lambdasp$ satisfies $\hat R_n^+(\lambda,\lambdasp,\delta')<\alpha$ and $\hat R_n^+(\lambda,\cdot,\delta')$ is also non-increasing and continuous, $\hat R_n^+(\lambda,\lambda_*,\delta')<\alpha$. However, since $R(\lambda,\lambda_*)=\alpha$ (by continuity), the pointwise confidence bounds in Equation \ref{eq:pointwise-bound} ensures that this occurs with probability at most $1-\delta'$.
\end{proof}

Now that we have established how we can bound the performative and generalization error, we now prove the following important theorem, which can then be used to prove each point of Thm.~\ref{thm:risk-control}.

\begin{theorem}[Safety of Each Deployment]
\label{thm:safety-of-each-deployment}
Take the setting of Theorem \ref{thm:risk-control}. With probability $1-\delta$, $R(\hat\lambda_t,\hat\lambda_t)\leq \alpha$ is ensured for all $t$ in $0\leq t\leq T$ (here, we denote $\hat\lambda_0=\lambda_0$).
\end{theorem}

\begin{proof}
We establish this theorem by verifying the following inductive statement: on iteration $t$, $R(\hat\lambda_s)\leq\alpha$ for all $s$ in $0\leq s\leq t$ with probability $1-\delta t/\tilde T$.

To establish the base case, recall the definition of \(\lambda_{\text{safe}}\), from which we see $\ell(Z,\lambda_0)=\ell(Z,\lambda_{\text{safe}})=0<\alpha$, trivially satisfying $R(\lambda_0)<\alpha$ with probability $1$. Afterwards, for $t\geq 1$, there are two cases:

\textbf{Case 1:} $\hat\lambda_t=\hat\lambda_{t-1}$. From induction, $R(\hat\lambda_s)<\alpha$ for all $s$ in $0\leq s \leq t-1$ with probability $1 - \delta (t-1)/\tilde T$. Since $\hat\lambda_{t}=\hat\lambda_{t-1}$, $R(\hat\lambda_s)<\alpha$ for all $s$ in $0\leq s \leq t$ with probability $1-\delta t/\tilde T < 1 - \delta (t-1)/\tilde T$.

\textbf{Case 2:} $\hat\lambda_t<\hat\lambda_{t-1}$. This implies $\hat\lambda_{t} = \inf \{\lambda\in \Lambda\mid V(\hat\lambda_{t-1},\lambda,\delta)\le \alpha\}$. The following chain of inequalities holds with probability \(1-\delta/\tilde T\):
\begin{align*}
R\bigl(\hat\lambda_t\bigr)
&\;=\;
   R\bigl(\hat\lambda_{t-1}, \hat\lambda_t\bigr)
   \;+\;
   R\bigl(\hat\lambda_t, \hat\lambda_t\bigr)
   \;-\;
   R\bigl(\hat\lambda_{t-1}, \hat\lambda_t\bigr)\\
&\;\le\;
  R\bigl(\hat\lambda_{t-1}, \hat\lambda_t\bigr)
   \;+\;
   \gamma M \,\bigl|\hat\lambda_t-\hat\lambda_{t-1}\bigr|\\
&\;\le\;
  \hat R_n^+\bigl(\hat\lambda_{t-1}, \hat\lambda_t,\delta/\tilde T \bigr)
  \;+\;
   \gamma M \,\bigl|\hat\lambda_t-\hat\lambda_{t-1}\bigr|\\
&\;=\;
  \alpha
  \;-\;
  \tau\,(\hat\lambda_{t-1}-\hat\lambda_t)
  \;+\;
  \gamma M \,\bigl|\hat\lambda_t-\hat\lambda_{t-1}\bigr|\\
&\;=\;
  \alpha
  \;-\;
  \bigl(\tau - \gamma M\bigr)\,\bigl|\hat\lambda_t-\hat\lambda_{t-1}\bigr|\\
&\;\le\;
  \alpha
\end{align*}

where the second line follows from Lemma~\ref{lemma:sensitivity-quantile}, the third line from Proposition~\ref{prop:prc-core}, and the fourth line follows from the continuity of the loss function $\ell(z,\lambda)$ in $\lambda$ for all $z$.

Since by induction $R(\hat\lambda_s)<\alpha$ for all $s$ in $0\leq s\leq t-1$ with probability $1-\delta (t-1)/\tilde T$, taking the union bound verifies that risk is control for all $s$ in $0\leq s\leq t$ with probability $1-\delta t/\tilde T$. Since $T\leq \tilde T$, this completes the proof.
\end{proof}

Having established Thm.~\ref{thm:safety-of-each-deployment}, we can now prove each part of Thm.~\ref{thm:risk-control}, which we do in the following three proofs. Note that we denote $M=\left[ \int_0^1 |\psi(p)|^udp \right]^{1/u}$ as in Lemma~\ref{lemma:sensitivity-quantile}, and $M=1$ for the case of expected risk.

\begin{proof}[Proof of part (i)]

We verify the following inductive statement: with probability at least  $1-\delta t/{\tilde T}$, $R(\hat\lambda_{s-1},\hat\lambda_{s})\leq\alpha$ for all $s$ in $0<s\leq t$. Once verified, since PRC takes at most $\tilde T$ iterations to finish, the theorem follows.

\textbf{Base case: } $t=1$

If $\hat\lambda_{1}=\lambda_0$, then with probability $1$, $R(\lambda_0,\hat\lambda_{1})=0<\alpha$ by the definition of $\lambda_{\text{safe}}$. Otherwise, $\hat\lambda_{1} = \inf \left\{ \lambda \in \Lambda \mid V(\lambda_0, \lambda, \delta) \leq \alpha \right\}$
, and with probability $1-\delta/\tilde T$,

\begin{align*}
R\bigl(\lambda_0,\hat\lambda_{1}\bigr)
&\;\le\;
   \hat R_n^+ \bigl(\lambda_0, \hat\lambda_{1},\delta/\tilde T\bigr)\\
&\;=\;
  \alpha
  \;-\;
  \tau M\,(\lambda_0-\hat\lambda_{1})\\
&\;\leq\;
  \alpha
\end{align*}

The first inequality follows from Proposition \ref{prop:prc-core}, and the last follows because if $\hat\lambda_{1} = \inf \left\{ \lambda \in \Lambda \mid V(\hat{\lambda}_{t-1}, \lambda, \delta) \leq \alpha \right\}$, then $\hat\lambda_{1}<\lambda_0$.

\textbf{Inductive step: } $t>1$

Assume $R(\hat\lambda_{s-1},\hat\lambda_{s})\leq \alpha$ for all $s$ in $0< s \leq t-1$ with probability $1-\delta (t-1)/\tilde T$. On iteration $t$, either $\hat\lambda_t=\hat\lambda_{t-1}$ or $\hat\lambda_t<\hat\lambda_{t-1}$. If $\hat\lambda_t=\hat\lambda_{t-1}$, we appeal to Thm.  \ref{thm:safety-of-each-deployment} to ensure that $R(\hat\lambda_{t-1},\hat\lambda_t)\leq \alpha$ with probability $1-\delta t/\tilde T$. The error rate of $\delta(t-1)/\tilde T$ from the inductive step does not accumulate with that of Theorem \ref{thm:safety-of-each-deployment} because both rely on the same pointwise UCB $\hat R_n^+(\hat\lambda_{t-1},\cdot)$ per iteration in the same trajectory. Hence, their maximum becomes the new error rate $\delta t / \tilde T$.

Otherwise, $\hat\lambda_t<\hat\lambda_{t-1}$. In this case, we follow similar steps to the base case. With probability $1-\delta/\tilde T$,

\begin{align*}
R\bigl(\hat\lambda_{t-1},\hat\lambda_t\bigr)
&\;\le\;
   \hat R_n^+ \bigl(\hat\lambda_{t-1}, \hat\lambda_t,\delta/\tilde T\bigr)\\
&\;=\;
  \alpha
  \;-\;
  \tau M\,(\hat\lambda_{t-1}-\hat\lambda_t)\\
&\;\leq\;
  \alpha
\end{align*}

Taking the union bound with the inductive assumption, we complete the proof.
\end{proof}

\begin{proof}[Proof of part (ii)]

Appealing to Thm. \ref{thm:safety-of-each-deployment} completes this part of the proof.

\end{proof}

\begin{proof}[Proof of part (iii)]

Algorithm \ref{alg:main} returns $\hat\lambda_T$ as the final risk-controlling conservativeness parameter. From Equation \ref{eq:pointwise-bound}, with probability $1-\delta'=1-\delta/\tilde T$,
\[
\hat  R_n\bigl(\hat\lambda_{T-1}, \hat\lambda_{T}\bigr) - c(n,\delta') \leq  R\bigl(\hat\lambda_{T-1}, \hat\lambda_{T}\bigr)\leq  \hat R_n\bigl(\hat\lambda_{T-1}, \hat\lambda_{T}\bigr) + c(n,\delta').
\]

There are two cases to consider at the algorithm's end.

\textbf{Case 1:} $\hat\lambda_{T}=\hat\lambda_{T-1}$. This implies 
$\hat\lambda_{T-1} \leq \inf \{\lambda\in \Lambda\mid V(\hat\lambda_{t-1},\lambda,\delta)\le \alpha\}:=\hat\lambda^{\text{mid}}_T$.

\begin{align*}
R\bigl(\hat\lambda_T\bigr)
&= 
   R\bigl(\hat\lambda_{T-1}\bigr) \\
&\geq 
   \hat R_n^-\bigl(\hat\lambda_{T-1},\hat\lambda_{T-1},\delta/\tilde T\bigr) \\
&=
   \hat R_n^+\bigl(\hat\lambda_{T-1},\hat\lambda_{T-1},\delta/\tilde T\bigr) - 2c(n,\delta/\tilde T) \\
&\geq
   \hat R_n^+\bigl(\hat\lambda_{T-1},\hat\lambda^{\text{mid}}_T,\delta/\tilde T\bigr) - 2c(n,\delta/\tilde T) \\
&=
   \alpha - \tau(\hat\lambda_{T-1}-\hat\lambda^{\text{mid}}_T) - 2c(n,\delta/\tilde T) \\
&\geq
   \alpha - 2c(n,\delta/\tilde T) \\
\end{align*}

where the first inequality is with probability $1-\delta/\tilde T$ and the second inequality follows from the continuity and monotonicity of $\hat R_n(\lambda,\cdot)$ for all $\lambda$.

\textbf{Case 2:} $\hat\lambda_{T-1}-\Delta\lambda\leq \hat\lambda_{T}<\hat\lambda_{T-1}$. Recall that by Lemma \ref{lemma:sensitivity-quantile}, we can bound the performative error:

\[
\left| R\bigl(\hat\lambda_{T}, \hat\lambda_{T}\bigr)
   \;-\;
   R\bigl(\hat\lambda_{T-1}, \hat\lambda_{T}\bigr) \right|\leq\gamma M|\hat\lambda_{T}-\hat\lambda_{T-1}|.
\]

With probability $1-\delta'=1-\delta/\tilde T$,

\begin{align*}
R(\hat\lambda_{T})
&= R(\hat\lambda_{T-1},\hat\lambda_{T}) + \left[ R(\hat\lambda_{T},\hat\lambda_{T}) - R(\hat\lambda_{T-1},\hat\lambda_{T}) \right]\\
&\geq R(\hat\lambda_{T-1},\hat\lambda_{T}) - \gamma M |\hat\lambda_{T-1} - \hat\lambda_{T}| \\
&\geq \hat R_n^-(\hat\lambda_{T-1},\hat\lambda_{T},\delta') -  \gamma M |\hat\lambda_{T-1} - \hat\lambda_{T}|\\
&=\hat R_n^+(\hat\lambda_{T-1},\hat\lambda_{T},\delta') -  \gamma M |\hat\lambda_{T-1} - \hat\lambda_{T}| - 2c(n,\delta')\\
&= \alpha-\tau(\hat\lambda_{T-1}-\hat\lambda_{T}) -  \gamma M |\hat\lambda_{T-1} - \hat\lambda_{T}| - 2c(n,\delta')\\
&\geq \alpha-2\tau(\hat\lambda_{T-1}-\hat\lambda_{T}) - 2c(n,\delta')\\
&\geq \alpha - 2\tau\Delta\lambda - 2c(n,\delta')\\
&= \alpha-\Delta\alpha
\end{align*}

where we set $\Delta\lambda$ such that $\Delta\alpha= 2\tau\Delta\lambda + 2c(n,\delta')$. Doing so allows us to achieve $\Delta\alpha$ tightness with probability $1-\delta$.

\end{proof}
\subsection{Quantile extension proofs}\label{app:quantile-ext-proofs}

\paragraph{The risk measure's lower bound corresponds to the CDF upper bound.} Assume we have access to lower and upper CDF confidence bounds such that for all $\lambda\in\{ \hat\lambda_t \}_{t=0}^T$ and $\lambda'\in\Lambda$,

$$
\mathbb{P}\left( \hat F_{L,n}(w \mathbin{;} \lambda,\lambda',\delta')\leq  F(w \mathbin{;} \lambda,\lambda')\leq \hat F_{U,n}(w \mathbin{;} \lambda,\lambda',\delta')   \right) \geq 1-\delta' \quad \forall w
$$

We rewrite the above as $\mathbb{P}\left( \hat F_{L,n} \preceq  F\preceq \hat F_{U,n} \right) \geq 1-\delta'$, dropping the $\lambda$ and $\lambda'$ for the rest of this section. The following property establishes that $\hat F_{L,n} \preceq  F\preceq \hat F_{U,n}$ implies $\hat F^{-1}_{U,n} \preceq F^{-1} \preceq \hat F^{-1}_{L,n}$.

\begin{proposition}
If $F\preceq G$, then $F^{-1}\succeq G^{-1}$.
\end{proposition}
\begin{proof}
Recall the generalized inverse of a CDF \(H\) is
\[
H^{-1}(p)=\inf\{x\in\mathbb R: H(x)\ge p\},\quad p\in(0,1).
\]
Since \(F(x)\le G(x)\) for all \(x\), we have $\{x: G(x)\ge u\}\;\subseteq\;\{x: F(x)\ge u\}$, so
\[
F^{-1}(u)
=\inf\{x: F(x)\ge u\}
\;\ge\;
\inf\{x: G(x)\ge u\}
=G^{-1}(u).
\]
Hence \(F^{-1}\succeq G^{-1}\).
\end{proof}

Applying a weighting function does not change this relationship. Hence,

$$
\mathbb{P}\left(\hat R^-_{\psi,n}\psi(\hat\lambda_{t-1},\lambda',\delta') \leq R_\psi(\hat\lambda_{t-1},\lambda') \leq \hat R^+_{\psi,n}(\hat\lambda_{t-1},\lambda',\delta') \right) \geq 1-\delta'
$$

as desired.

\begin{proof}[Proof of Theorem~\ref{thm:quantile}]
This follows as a consequence of the proofs in App.~\ref{app:main-proofs}.
\end{proof}

\section{Experiments}
\label{app:exp}

In this section, we further detail our experimental setup and include other supporting results. Our credit scoring simulations were conducted on a CPU and can be finished in less than one hour on a standard laptop computer.

\subsection{Type II Error Control for Credit Scoring}
\label{app:cs-ev}

\paragraph{Set up and $\gamma$-sensitivity} Our data is sourced from the training split of a Kaggle credit scoring dataset \citep{GiveMeSomeCredit}, which contains roughly $150k$ data points, $10k$ of which ending up in a serious delinquency ($y=1$). We first sample $1.5k$ from each class to train a logistic classifier to function as our model $f(x)$. Because the rest of this dataset is so highly imbalanced, the expected loss would be close to $0$. For illustration purposes, we sample $8.5k$ instances from the negative class so that the proportion of positive and negative classes are balanced. Then, for each trajectory, we sample $n$ points without replacement to serve as the calibration set, with the remaining serving as the validation set on which the risk is evaluated on. To generate multiple trajectories, we sample different calibration and validation sets from the balanced subset.

We now prove that our simulated shift is $(\gamma,1,\ell(\cdot,\lambda))$-sensitive for all $\lambda$. Recall that our shift is defined as 
$$f(x(\lambda))=\begin{cases}
\max(0,f(x)-s) & \text{if } f(x)-s\leq 1-\lambda \\
f(x) & \text{otherwise}.
\end{cases}$$
where $s(=0.3)$ is an arbitrary  parameter we define.

Consider $\lambda_1$ and $\lambda_2$ so that $1-\lambda_1\leq1-\lambda_2\leq1-\lambda_1+s$. Then, $f(x(\lambda_1)$ and $f(x(\lambda_2))$ are equivalent except in the narrow region $[1-\lambda_1+s,1-\lambda_2+s]$ (here, we assume we start with the same $x$ and that the region is inside $[0,1]$). Let $p$ be the probability of a positive case, and let $C$ be the maximum value of the PDF of predictions $f(x)$ on the base distribution given $y=1$. It is possible that the threshold $\lambda$ causes a discrepancy of $1$ in the loss for these data points in the population by flipping the prediction, e.g., $\lambda=\lambda_1$. Hence, $W_1(\cD_{\ell(\cdot,\lambda)}(\lambda_1), \cD_{\ell(\cdot,\lambda)})\leq pC|\lambda_1-\lambda_2|$ for all $\lambda$.

\begin{figure}
    \centering
    \includegraphics[width=0.5\linewidth]{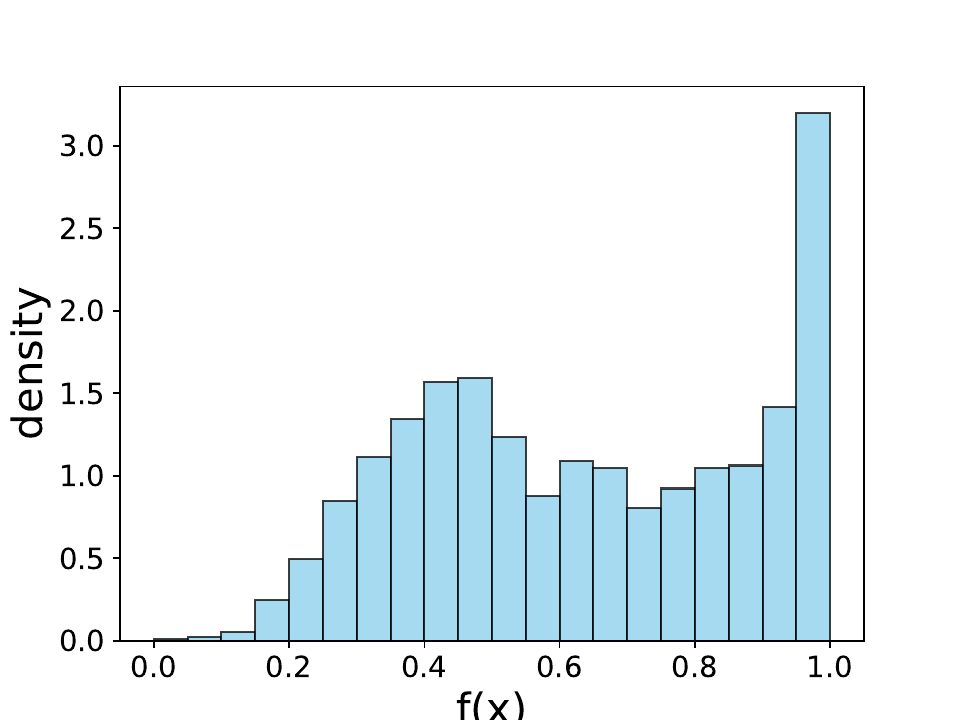}
    \caption{Histogram of predictions $f(x)$ on the base distribution.}
    \label{fig:density-plot}
\end{figure}

Now consider when $\lambda_1+s<\lambda_2$. If we use a threshold $\lambda$ so that $\lambda_1+s\leq \lambda \leq \lambda_2$, the original loss distribution and that induced by $\lambda_1$ are the same. However, for the distribution induced by $\lambda_2$, the original predictions falling in the region $[1-\lambda,1-\lambda+s]$ flip from $0$ to $1$. Similar to above, this implies that $W_1(\cD_{\ell(\cdot,\lambda)}(\lambda_1), \cD_{\ell(\cdot,\lambda)})\leq pCs$ for all $\lambda$. 

Combining the two cases above, we have $W_1(\cD_{\ell(\cdot,\lambda)}(\lambda_1), \cD_{\ell(\cdot,\lambda)})\leq pC\min(s,|\lambda_1-\lambda_2|)\leq pC|\lambda_1-\lambda_2|$ for all $\lambda$, achieving $\gamma$-sensitivity with $\gamma=pC$. We experimentally verify the values of $C$ and $p$ to be $3.20$ and $0.432$, respectively. Importantly, we calculate $C$ using the balanced subset rather than the training set. We use $20$ bins for the histogram (see Fig.~\ref{fig:density-plot}). Hence, $\gamma=pC\leq 1.38$.

\paragraph{Further experimental discussion}

Here, we include supporting plots to Fig.~\ref{fig:credit-scoring-expected-risk} and further discuss our experimental results. Fig.~\ref{fig:failure-prob} shows the proportion of runs with $\alpha-\Delta\alpha\leq R(\hat\lambda_T)\leq\alpha$, and Fig.~\ref{fig:loss_vs_iteration-all-rest-taus} displays similar plots to the bottom left plot of Fig.~\ref{fig:credit-scoring-expected-risk} for different values of $\tau$.

Fig.~\ref{fig:failure-prob} plots the failure probability for different values of $\tau$, defined as the proportion of trajectories where the final $\hat\lambda_T$ does not satisfy $\alpha-\Delta\alpha \leq R(\hat\lambda_T)\leq\alpha$. For $\tau=1$ and $\tau=2$ close to $\gamma\leq1.38$, the failure probability is essentially zero (much less than $\delta$), which some may argue constitutes loose coverage. While it is true that PRC's tightness guarantees are looser than those prevalent in the risk control literature, if we consider our unique setting of applying risk control under performative shift, the guarantees are actually quite tight, driven by three main factors. One, our algorithm needs to protect \textit{every iteration} for unknown functions $R(\lambda)$. The risk could be adversarial and necessitate many iterations, all of which PRC needs to protect. Two, our calculation of $\gamma$ is a uniform bound across the iterations, necessitating a larger safety parameter $\tau$; for example, it is possible that the ``true" $\gamma$ is between $\tau=0.2$ and $\tau=0.5$, at which point the empirical failure rate crosses $\delta$. With the other parameters such as the number of samples $n$ held constant, a larger $\tau$ corresponds to a larger $\Delta\alpha$. Finally, we do not assume ``insider" access to the statistics of the first or future distributions. We plan for the worst-case in terms of sample mean and sample variance in computing the confidence widths. In light of these three factors intrinsic to the performative setting, our bounds are indeed quite tight.

\begin{figure}
    \centering
    \includegraphics[width=0.5\linewidth]{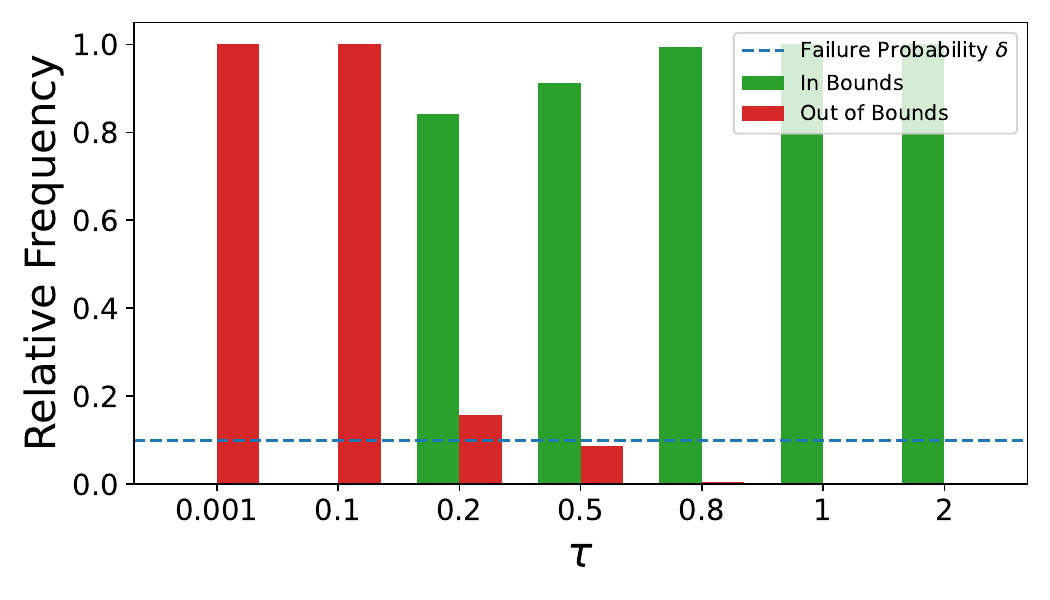}
    \caption{The relative frequency of runs that end with a $\hat\lambda_T$ with $\alpha-\Delta\alpha\leq R(\hat\lambda_T)\leq\alpha$.}
    \label{fig:failure-prob}
\end{figure}

\begin{figure}
   \centering
    \begin{subfigure}[b]{0.5\textwidth}
        \includegraphics[width=\textwidth]{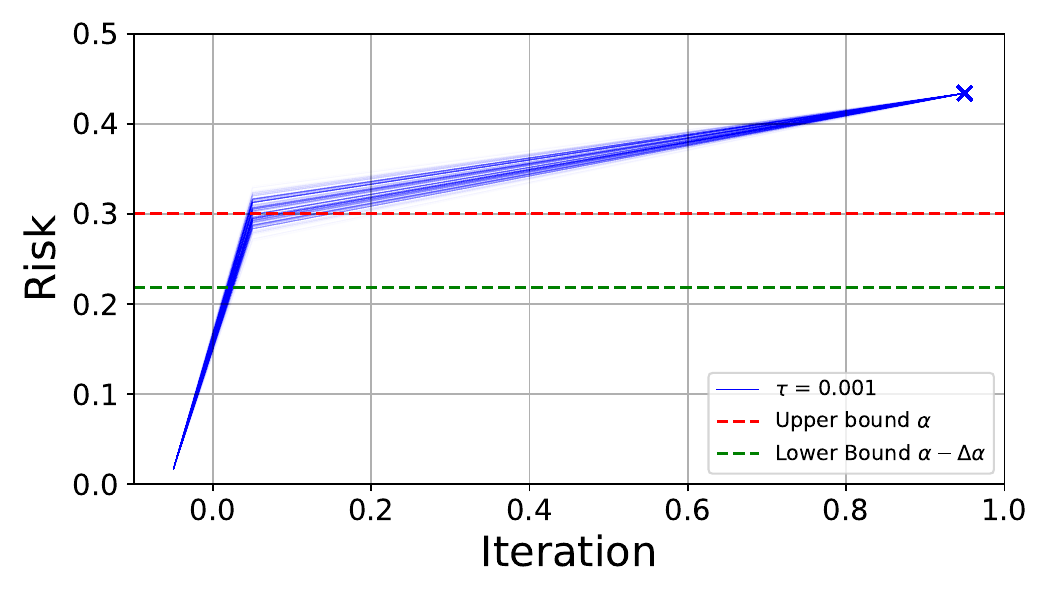}
    \end{subfigure}%
    \hfill
    \begin{subfigure}[b]{0.5\textwidth}
        \includegraphics[width=\textwidth]{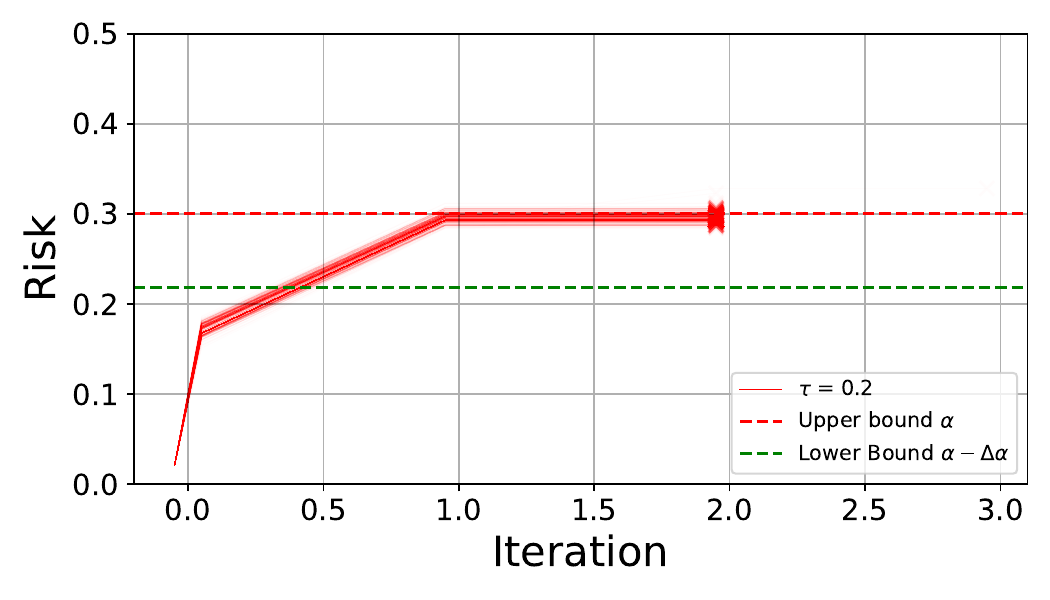}
    \end{subfigure}%

    \vspace{\baselineskip} % Vertical space between the two rows

    \begin{subfigure}[b]{0.5\textwidth}
        \includegraphics[width=\textwidth]{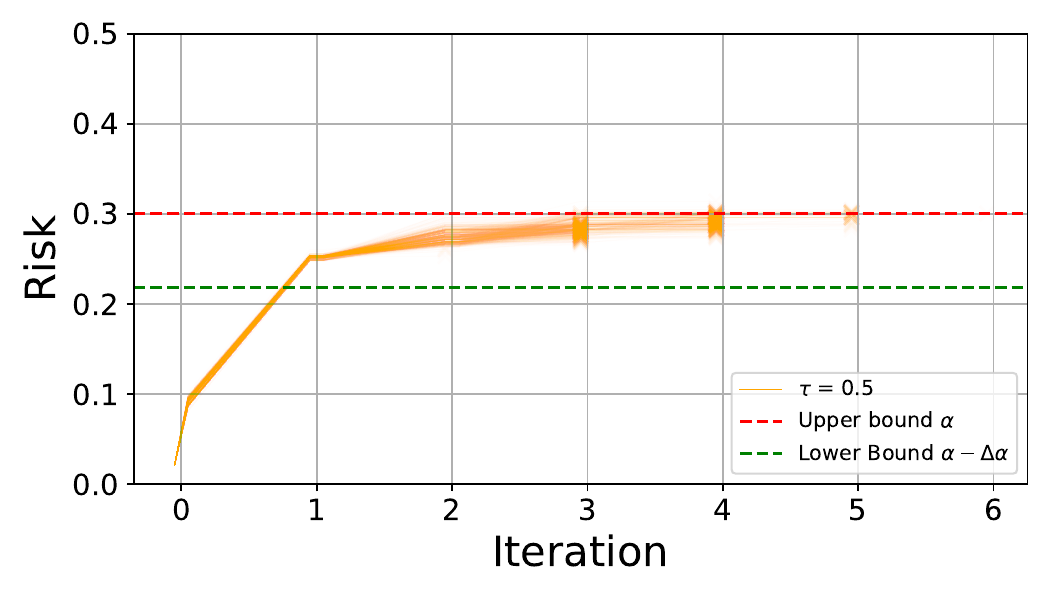}
    \end{subfigure}%
    \hfill
    \begin{subfigure}[b]{0.5\textwidth}
        \includegraphics[width=\textwidth]{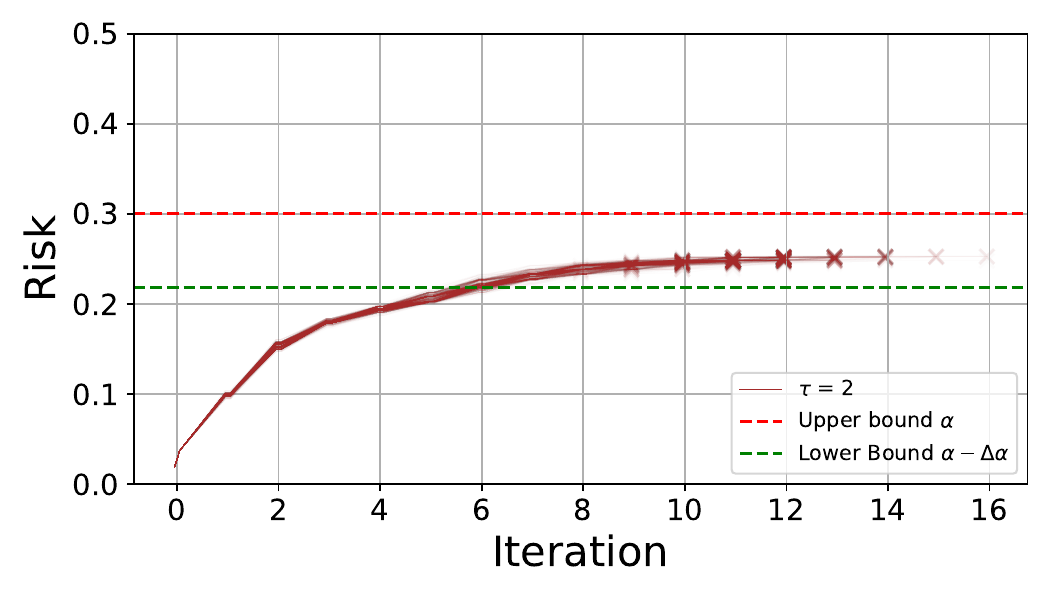}
    \end{subfigure}
    \caption{Analog of the bottom left plot in Fig.~\ref{fig:credit-scoring-expected-risk} for $\tau\in\{0.001,0.2,0.5,2\}$.}
    \label{fig:loss_vs_iteration-all-rest-taus}
\end{figure}

\subsection{Quantile Risk Control in Credit Scoring}
\label{app:cs-cvar}

\paragraph{Set up, $\gamma$-sensitivity, and the confidence width}

For this experiment, we adopt the same setting as in expected risk and use the same shift technique. However, after training the logistic classifier, \textit{we avoid any subsampling,} using the rest of the dataset. Hence, while $C$ remains the same, $p$ drops from $0.432$ to $0.064$, reducing $\gamma$ to $\gamma=pC\leq0.205$.

To calculate the confidence width, we follow a recipe based on a version of CLT that is applicable to quantile-based risk measures. The recipe calls for knowledge of the loss CDF. Recall that our loss is given by $\ell(y, \cT_\lambda(x))=yU[0,1](1-\cT_{\lambda,\epsilon}(x))$; we make a minor approximation to this loss function by taking $\epsilon\rightarrow 0$, so that the loss is binary. Denote by $p'$ the proportion of the population that has $y(1-\cT_{\lambda,0}(x))=1$. Because at most $p$ of the population has $y=1$, we know that $p'\leq p$, which holds for any given distribution-inducing $\lambda$ and threshold $\lambda'$. Knowing $p'$ allows us to derive the CDF and its inverse, which are piecewise linear because of the uniform distribution.

\begin{align*}
F(w) &= \begin{cases}
0 & \text{when } w < 0 \\
1-p'+p'w & \text{when } 0 \le w \le 1 \\
1 & \text{when } w > 1
\end{cases} \\
F^{-1}(r) &= \begin{cases}
0 & \text{when } 0 < r \le 1-p \\
1-\frac{1-r}{p} & \text{when } 1-p < r \le 1
\end{cases}
\end{align*}

By Theorem 22.3 in \citet{Vaart_1998}, we have
$$
\sqrt{n}(R_\psi(\hat F_n) - R_\psi(F)) \to_D N(0, \sigma^2(\psi, F)),
$$
where
$$
\sigma^2(\psi, F) = \iint (F(r \wedge \tilde r) - F(r)F(\tilde r)) \psi(F(r))\psi(F(\tilde r)) \,dr\,d\tilde r.
$$

Since the $\beta$-CVaR metric is bounded everywhere and has one discontinuous point, we can apply this theorem (appealing to dominated convergence). We split the calculation of $\sigma^2(\psi, F)$ into two cases.

\paragraph{Case I: $\beta\leq1-p'$}

\begin{align*}
\sigma^2(\psi,F) &= \frac{1}{(1-\beta)^2} \int_0^\infty \int_0^\infty F(r\wedge \tilde r) - F(r)F(\tilde r) dr d\tilde r\\
&=\frac{1}{(1-\beta)^2} \int_0^1 \int_0^1 F(r\wedge \tilde r) - F(r)F(\tilde r) dr d\tilde r
\end{align*}

Using $F(r)=1-p'+p'r$ gives us
$$
\sigma^2(\psi,F) = \frac{1}{(1-\beta)^2} \frac{1}{12}(4-3p')p'
$$

\paragraph{Case II: $\beta>1-p'$}

We require $F(r)=1-p'+p'r \geq \beta$, or $r \geq (\beta-1+p)/p$ for the weighting function to be nonzero. Similarly to above, the upper bound can be changed from $\infty$ to $1$ because $F(r)=1$ for $r\geq 1$. The lower bound is then changed to $a=1-\frac{1-\beta}{p}$.
\begin{align*}
\sigma^2(\psi,F) &= \frac{1}{(1-\beta)^2} \int_a^1 \int_a^1 F(r\wedge \tilde r) - F(r)F(\tilde r) dr d\tilde r \\
&= \frac{2}{(1-\beta)^2} \int_a^1 \int_r^1 F(r) - F(r) F(\tilde r) dr d\tilde r\\
&= \frac{2}{(1-\beta)^2} \int_a^1 \int_r^1 F(r) [1 - F(\tilde r) ] dr d\tilde r\\
&=\frac{2}{(1-\beta)^2} \int_a^1 (1-p'+p'r)(\frac{1}{2}p' -p'r + \frac{1}{2}p'r^2) dr \\
&= \frac{1-\beta}{3\beta+1}{12p'^2}
\end{align*}

We examine $\beta=0.9$ ($90\%$-CVaR) and when the range of $p'$ is $[0,p]$, where $\beta\leq1-p$. Taking the maximum $\sigma^2(\psi,F)$, which occurs precisely when $p'=p$, we obtain the variance, which allows us to construct the following confidence width that applies for any choice of shift-inducing $\lambda$ and threshold $\lambda'$:
$$
c(n,\delta')=\Psi^{-1}(1-\delta'/2) \frac{1}{1-\beta}\sqrt{\frac{(4-3p)p}{12n}}.
$$

\paragraph{Further experimental results}

Here, we include supporting plots to Fig.~\ref{fig:credit-scoring-cvar}. Fig.~\ref{fig:failure-prob-quantile} shows the proportion of runs with $\alpha-\Delta\alpha\leq R(\hat\lambda_T)\leq\alpha$, Fig.~\ref{fig:lambda-plots-quantile} shows sample trajectories for different values of $\tau$, analogous to the top two plots of Fig.~\ref{fig:credit-scoring-expected-risk}, and Fig.~\ref{fig:loss_vs_iteration-all-rest-taus-quantile} displays the risk trajectory for different selections of $\tau$.

\begin{figure}
    \centering
    \includegraphics[width=0.5\linewidth]{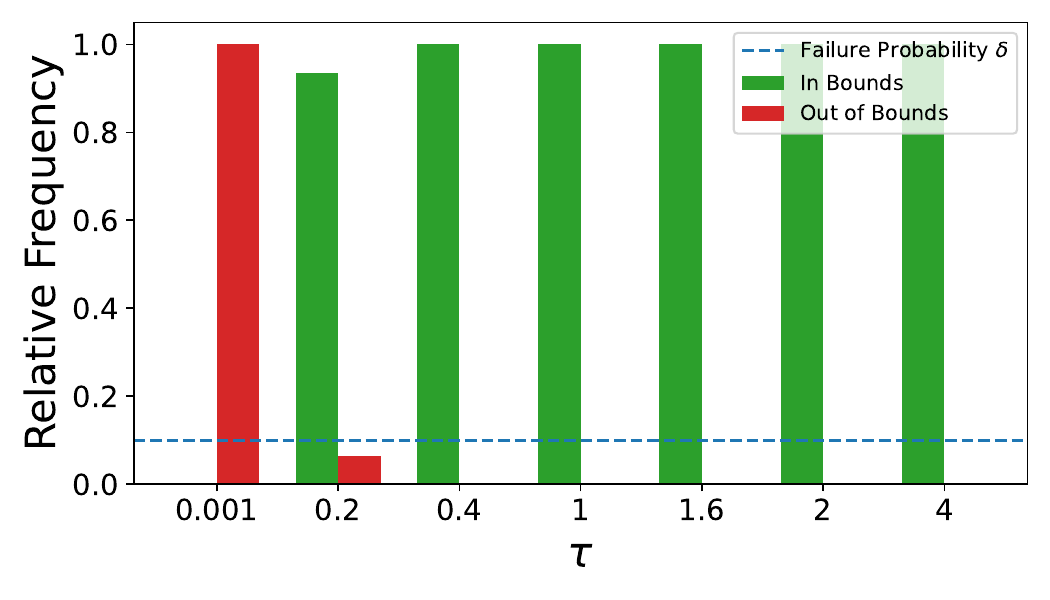}
    \caption{The relative frequency of runs that end with a $\hat\lambda_T$ with $\alpha-\Delta\alpha\leq R(\hat\lambda_T)\leq\alpha$.}
    \label{fig:failure-prob-quantile}
\end{figure}

\begin{figure}
   \centering
    \begin{subfigure}[b]{0.5\textwidth}
        \includegraphics[width=\textwidth]{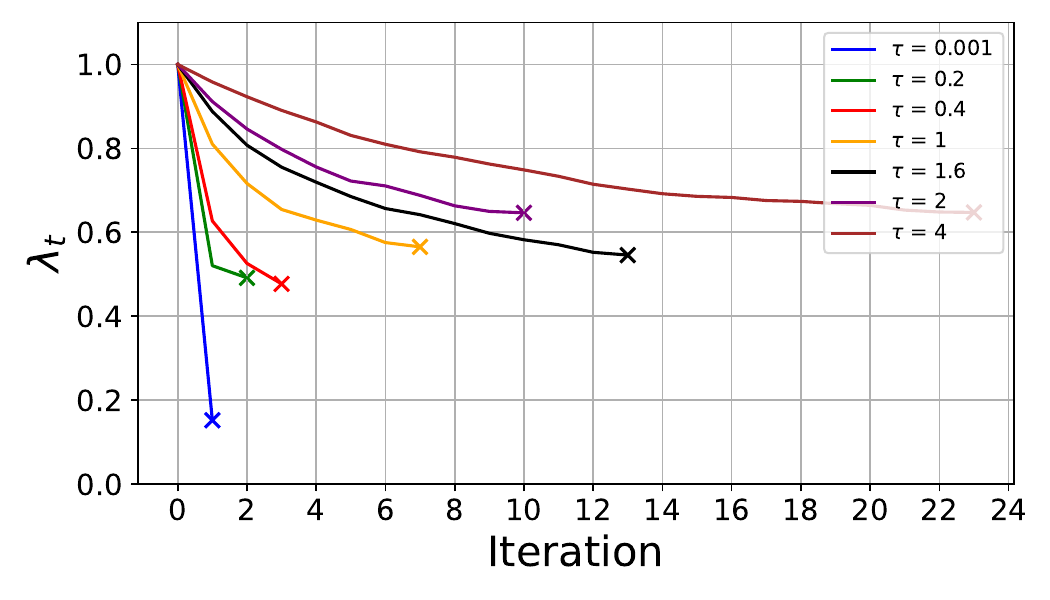}
    \end{subfigure}%
    \hfill
    \begin{subfigure}[b]{0.5\textwidth}
        \includegraphics[width=\textwidth]{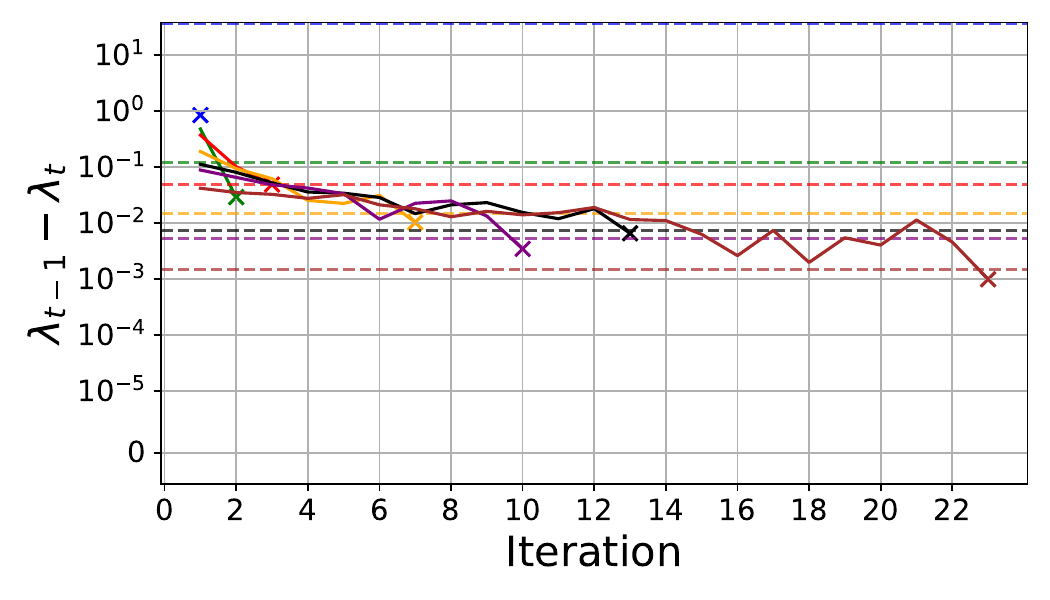}
    \end{subfigure}%
    \caption{Analog of the bottom left plot in Fig.~\ref{fig:credit-scoring-expected-risk} for $\tau\in\{0.001,0.2,0.4,1,1.6,2,4\}$.}
    \label{fig:lambda-plots-quantile}
\end{figure}

\begin{figure}
   \centering
    \begin{subfigure}[b]{0.5\textwidth}
        \includegraphics[width=\textwidth]{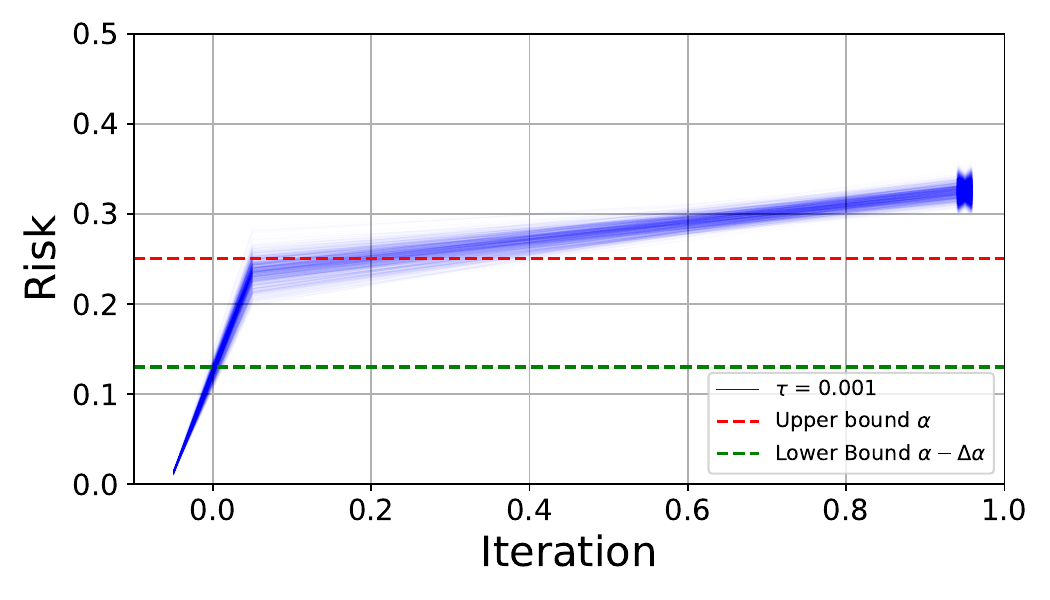}
    \end{subfigure}%
    \hfill
    \begin{subfigure}[b]{0.5\textwidth}
        \includegraphics[width=\textwidth]{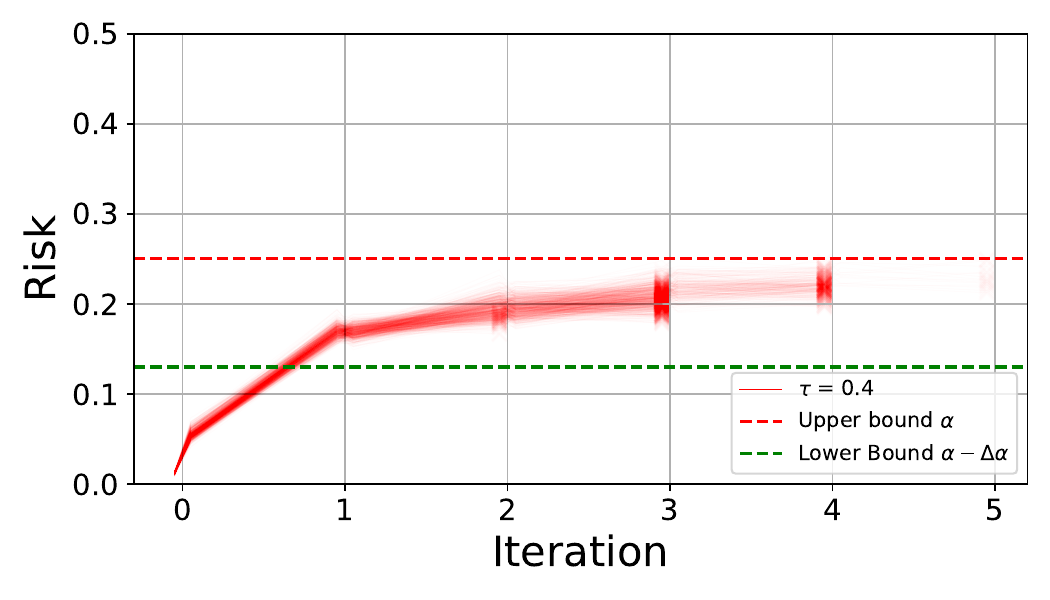}
    \end{subfigure}%

    \vspace{\baselineskip} % Vertical space between the two rows

    \begin{subfigure}[b]{0.5\textwidth}
        \includegraphics[width=\textwidth]{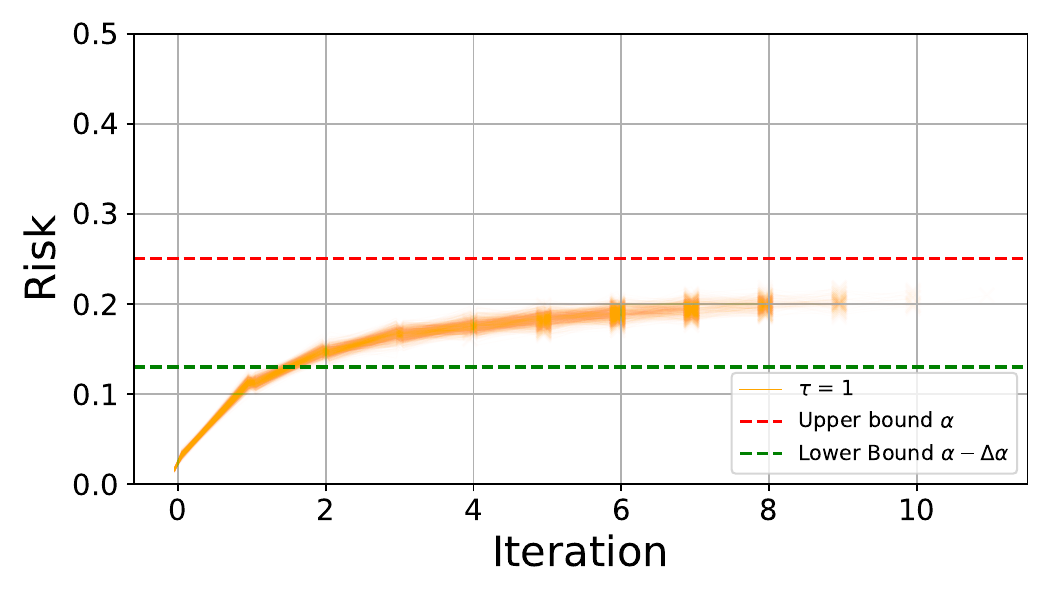}
    \end{subfigure}%
    \hfill
    \begin{subfigure}[b]{0.5\textwidth}
        \includegraphics[width=\textwidth]{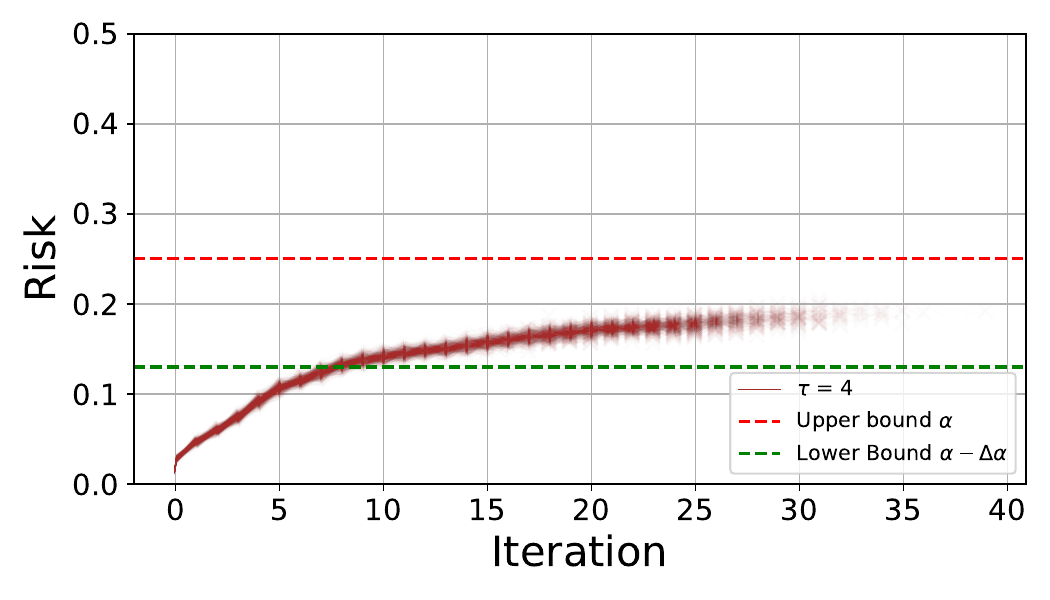}
    \end{subfigure}
    \caption{Analog of the bottom left plot in Fig.~\ref{fig:credit-scoring-expected-risk} for the $90\%$-CVaR risk measure and $\tau\in\{0.001,0.4,1,4\}$.}
    \label{fig:loss_vs_iteration-all-rest-taus-quantile}
\end{figure}

\end{document}